\documentclass{article}

\usepackage[preprint]{neurips_2021}

\usepackage[utf8]{inputenc}

\title{
On Margin-Based Cluster Recovery \\ with Oracle Queries
}
\usepackage{caption}
\usepackage{wrapfig}
\usepackage{enumitem}
\usepackage[utf8]{inputenc} 
\usepackage[T1]{fontenc}    
\usepackage{hyperref}       
\usepackage{url}            
\usepackage{booktabs}       
\usepackage{microtype}      
\usepackage{graphicx}
\usepackage[fleqn]{amsmath}
\usepackage{amsfonts,amsthm,amssymb}
\usepackage[noend]{algpseudocode}
\usepackage{algorithm}
\usepackage{xcolor}
\usepackage{nicefrac}
\usepackage{tikz}
\usetikzlibrary{calc,positioning,fadings}

\newcommand{\scI}{\mathcal{I}}
\newcommand{\scH}{\mathcal{H}}

\newcommand{\scO}{\mathcal{O}}

\newcommand{\scA}{\mathcal{A}}
\newcommand{\scP}{\mathcal{P}}

\newcommand{\scU}{\mathcal{U}}

\newcommand{\CoverNum}{\mathcal{N}} 
\newcommand{\PackNum}{\mathcal{M}} 

\newcommand{\field}[1]{\mathbb{#1}}
\newcommand{\R}{\field{R}}
\newcommand{\Nat}{\field{N}}
\newcommand{\E}{\field{E}}
\renewcommand{\Pr}{\field{P}}

\newcommand{\norm}[1]{ \left\|{#1}\right\| }

\newcommand{\ve}{\varepsilon}
\newcommand{\C}{\mathcal{C}}
\newcommand{\scX}{\mathcal{X}}

\newcommand{\Hyp}{\mathcal{H}}

\renewcommand{\hat}{\widehat}
\renewcommand{\bar}{\overline}
\renewcommand{\epsilon}{\ve}

\newtheorem{lemma}{Lemma}

\newtheorem{theorem}{Theorem}

\newtheorem{definition}{Definition}

\DeclareMathOperator{\poly}{poly}

\newcommand{\scq}{\textsc{scq}}

\newcommand{\dotp}[1]{\left\langle{#1}\right\rangle}
\newcommand{\vol}{\operatorname{vol}}

\newcommand{\conv}{\operatorname{conv}} 

\renewcommand{\norm}[2]{\|#2\|_{#1}}
\newcommand{\orig}{\boldsymbol{0}}

\newcommand{\vcdim}{\ensuremath{\operatorname{vc-dim}}}
\newcommand{\diam}{\phi}

\newcommand{\sldim}{\operatorname{sl}}
\newcommand{\cosl}{\operatorname{cosl}}

\newcommand{\co}{\operatorname{co}}

\newcommand{\CoolAlgo}{\textsc{CheatRec}} 
\newcommand{\MarginAlgo}{\textsc{mRec}} 
\newcommand{\ExpandHull}{\textsc{HullTrick}}

\usepackage{thmtools, thm-restate}

\author{Marco Bressan
\\
Dept.\ of CS, Univ.\ of Milan, Italy
\\
marco.bressan@unimi.it
\And
Nicolò Cesa-Bianchi
\\
DSRC \& Dept.\ of CS, Univ.\ of Milan, Italy
\\
nicolo.cesa-bianchi@unimi.it
\AND
Silvio Lattanzi
\\ 
Google
\\
silviol@google.com
\And
Andrea Paudice
\\
Dept.\ of CS, Univ.\ of Milan, Italy \& \\
Istituto Italiano di Tecnologia, Italy
\\
andrea.paudice@unimi.it
}

\begin{document}

\maketitle

\begin{abstract}
We study an active cluster recovery problem where, given a set of $n$ points and an oracle answering queries like ``are these two points in the same cluster?'', the task is to recover exactly all clusters using as few queries as possible. We begin by introducing a simple but general notion of margin between clusters that captures, as special cases, the margins used in previous works, the classic SVM margin, and standard notions of stability for center-based clusterings. Under our margin assumptions we design algorithms that, in a variety of settings, recover all clusters exactly using only $\scO(\log n)$ queries. For the Euclidean case, $\R^m$, we give an algorithm that recovers \emph{arbitrary} convex clusters, in polynomial time, and with a number of queries that is lower than the best existing algorithm by $\Theta(m^m)$ factors. For general pseudometric spaces, where clusters might not be convex or might not have any notion of shape, we give an algorithm that achieves the $\scO(\log n)$ query bound, and is provably near-optimal as a function of the packing number of the space. Finally, for clusterings realized by binary concept classes, we give a combinatorial characterization of recoverability with $\scO(\log n)$ queries, and we show that, for many concept classes in Euclidean spaces, this characterization is equivalent to our margin condition. Our results show a deep connection between cluster margins and active cluster recoverability.

\end{abstract}

\section{Introduction}
This work investigates the problem of exact cluster recovery using oracle queries, in the well-known framework introduced by~\citet{ashtiani2016clustering}. We are given a set $X$ of $n$ points from some domain $\scX$ (e.g., from the Euclidean $m$-dimensional space $\R^m$) and an oracle answering to same-cluster queries of the form ``are these two points in the same cluster?'' or, equivalently, to label queries of the form ``which cluster does this point belong to?''. The oracle answers are consistent with some clustering $\C=(C_1,\ldots,C_k)$ of $X$ unknown to the algorithm, where $k$ is a fixed constant. The goal is to design an algorithm that recovers $\C$ deterministically by using as few queries as possible.

Clearly, if there are no restrictions on $\C$, then any algorithm needs $n$ queries in the worst case. Thus, the question is what assumptions on $\C$ yield query-efficient algorithms. Since a good clustering is often thought of as having well-separated clusters, a natural assumption is that $\C$ satisfies some margin property; and previous work shows precisely that some margin properties yield cluster recovery algorithms that achieve the ``gold standard'' bound of $\scO(\log n)$ queries. The first such algorithm appeared in~\citet{ashtiani2016clustering} for the Euclidean case (i.e., when $X \subseteq \R^m$), with the following result. If every cluster $C_i$ is separated from $X \setminus C_i$ by a ball that is centered in the center of mass of $C_i$, then $\scO(\log n)$ queries are sufficient to recover $\C$ with high probability, provided that, for some fixed $\gamma > 0$, every ball does not include points of other clusters, even if expanded by a factor of $1+\gamma$. The parameter $\gamma$ is called \emph{margin}, and the number of queries needed to recover $\C$ grows with $\nicefrac{1}{\gamma}$.

In a first attempt at generalizing this result,~\citet{BCLP20} showed that with $\scO(\log n)$ queries one can actually recover clusters with \emph{ellipsoidal} separators, at the price of a dependence on $\gamma$ and $m$ of approximately $(\nicefrac{m}{\gamma})^{m}$. Interestingly, this result was achieved via boosting of one-sided error learning, which works as follows. Suppose that, by making $\scO(1)$ queries, we could identify correctly (with zero mistakes) a constant fraction of the points in some cluster $C_i$. Then, we could label those points as~$i$, remove them from the dataset, and repeat. It is not hard to show that, after $\scO(\log n)$ rounds, we have correctly labeled all the input points with high probability. The difficult task is, of course, using only $\scO(1)$ queries to identify correctly a constant fraction of some cluster $C_i$ (this is called learning with one-sided error, since all points predicted to be in $C_i$ are indeed in $C_i$, while points predicted \emph{not} to be in $C_i$ can be anywhere). The key insight of~\citet{BCLP20} is that, if the clusters have margin $\gamma$ with respect to their ellipsoidal separators, then roughly $(\nicefrac{m}{\gamma})^m$ queries are sufficient. This leads to the following question: how much can this margin-based approach be extended?

In this work we provide several answers, revealing how margin-based cluster recovery and one-sided error learning are intimately connected. Our main contributions are as follows.
\begin{enumerate}[itemsep=1pt,parsep=1pt,leftmargin=12pt]
    \item We introduce a new notion of margin in Euclidean spaces, that we call ``convex hull margin'' (Definition~\ref{def:ch_margin}). This is a strict generalization of the margins of~\citep{BCLP20,ashtiani2016clustering} and of the usual SVM margin, and allows the clusters to have \emph{any shape whatsoever} as long as they are convex. Under the convex hull margin, we develop a novel technique for learning with one-sided error that we call \emph{convex hull expansion trick}. It essentially amounts to sampling many points from a single cluster and ``inflate'' their convex hull by a factor of $(1+\gamma)$. This simple technique (whose proof, however, is quite technical) yields an algorithm for exact cluster recovery in $\R^m$ which runs in polynomial time and makes $\scO(\log n)$ queries (Theorem~\ref{thm:cool}). The dependence of the query bound on $\gamma$ and $m$ is of order $(1+\nicefrac{1}{\gamma})^m$, which is significantly better than~\citep{BCLP20} and closer to the query complexity lower bound of order $(1+\nicefrac{1}{\gamma})^{m/2}$.
    \item We introduce a notion of cluster margin for general pseudometric spaces called \emph{one-versus-all margin} (Definition~\ref{def:ova_margin}). This notion of margin is strictly more general than convex hull margin, and captures as special cases some standard notions of stability for hard clustering problems, such as $k$-means or $k$-centers. We show that, if a clustering has one-versus-all margin, then it can be recovered with $\scO(\log n)$ queries by an algorithm that uses a purely learning-based approach (Theorem~\ref{thm:ova_margin}). The $\scO(\log n)$ query bound hides a dependence on the complexity of the pseudometric space expressed in terms of its packing number. We show that such a dependence is essentially optimal, thus characterizing the recoverability of clusterings in this setting.
    \item Finally, we show a new connection between margin-based learning and exact active cluster recoverability, when clusters are realized by some concept class $\Hyp$ (that is, when for each cluster $C_i$ there is a concept $h_i \in \Hyp$ such that $X \cap h_i = C_i$). We show that if a certain combinatorial parameter of $\Hyp$, called \emph{coslicing dimension}, is bounded, then one can learn clusterings with $\scO(\log n)$ label queries; otherwise, $\Omega(n)$ queries are needed in the worst case (Theorem~\ref{thm:kclass_learning}). Moreover, we show that if $\Hyp$ is \emph{any} concept class in $\R^m$ that is closed under affine transformations and well-behaved in a natural sense, then the clusters realized by $\Hyp$ are recoverable with $\scO(\log n)$ label queries if and only if they have positive one-versus-all margin (Theorem~\ref{thm:affine}).
\end{enumerate}
Our contributions have implications for active learning of binary and multiclass classifiers, too. More precisely, our $\scO(\log n)$ query bounds imply $\widetilde{\scO}(\log \nicefrac{1}{\epsilon})$ query bounds for pool-based active learning~\citep{McCallum98}, where $\epsilon$ is the generalization error. To see this, draw a set $X$ of $\Theta\big(\epsilon^{-1}(K \log \nicefrac{1}{\epsilon} + \log \nicefrac{1}{\delta})\big)$ unlabeled samples from the underlying distribution, where $K$ is the relevant measure of capacity (the VC-dimension or the Natarajan dimension), run our algorithms over $X$, and compute a hypothesis consistent with the recovered labeling $\C$. These types of reductions are standard in active learning, see for instance~\citep{Kane17}.

\paragraph{Additional related work.}
Same-cluster queries are natural to implement in crowd-sourcing systems, and for this reason they have been extensively studied both in theory~\citep{ailon2018approximate, ailon2017approximate, gamlath2018semi, huleihel2019same, mazumdar2017semisupervised, mazumdar2017clustering, NIPS2017_7054, saha2019correlation, vitale2019flattening} and in practice~\citep{firmani2018robust, gruenheid2015fault, verroios2015entity, verroios2017waldo}. Note that same-cluster queries are essentially equivalent to label queries, the basic mechanism of active learning \citep{MAL-037}.

Various notions of margin are central in both active learning and cluster recovery \citep{xu2004maximum,Balcan07margin,Balcan13-logconcave,Kane17,BCLP21-density}. Our coslicing dimension is similar to the slicing dimension of \citet{Kivinen95} and the star number of \citet{hanneke2015minimax}. Our arguments based on packing numbers are similar to those based on the inference dimension of \citet{Kane17} or the lossless sample compression of \citep{Hopkins21boundedmem}, as we cannot infer the label of a point only when it is far from already-labeled points. Combinatorial characterizations of multiclass learning have been also proposed in the passive case by \citet{bendavid1995characterizations,rubinstein2009shifting,Daniely14multiclass}. Other learning settings related to one-sided and active learning are RPU learning \citep{rivest1988learning} and perfect selective classification \citep{el2012active} --- see \citep{hopkins2019power} for a discussion.

\section{Preliminaries and notation}
\label{sec:setup}
The input is a pair $(X,O)$, where $X$ is a set of $n$ points from some domain $\scX$, and $O$ is a label oracle that, when queried on any $x \in X$, returns the cluster id $\C(x)$ of $x$. The oracle $O$ is consistent with a \emph{latent clustering} $\C=(C_1,\ldots,C_k)$ of $X$, i.e., a $k$-uple of pairwise disjoint sets whose union is $X$.\footnote{In line with previous works, we assume $k$ is fixed and known.} Note that we allow clusters to be empty. Our goal is to recover $\C$ exactly, by making as few queries as possible to $O$. Queries can be made adaptively, that is, the $j$-th point to be queried can be chosen as a function of the answers to the first $j-1$ queries. We express the number of queries as a function of $k$, $n$, and other parameters to be introduced later. This setting is essentially equivalent to the semi-supervised active clustering (SSAC) framework of \citet{ashtiani2016clustering}, where the oracle answers same-cluster queries $\scq(x,y)$ that, for any two points $x,y \in X$, return $\textsc{true}$ iff $\C(x)=\C(y)$. We use label queries instead of \scq\ queries only for simplicity. Indeed, any \scq\ query can be emulated with two label queries; conversely, the label of any point can be learned with $k$ \scq\ queries, up to a relabeling of the clusters. This implies that our bounds on the number of queries for cluster recovery hold for an \scq\ oracle as well, up to a multiplicative factor of $k$. 

We often assume a metric or a pseudometric $d$ over $\scX$ (a pseudometric allows two distinct points to have distance $0$). For any $X \subset \scX$, we denote by $\diam_d(X) = \sup_{x,x' \in X} d(x,x')$ the diameter of $X$ measured by $d$, and we define $\diam_d(\emptyset)=0$. For any two sets $U,S \subset \scX$, we denote by $d(U,S)=\inf_{x \in U,y \in S}d(U,S)$ their distance according to $d$, and we define $d(U,\emptyset)=\infty$. For any $X \subset \R^m$, we write $\conv(X)$ for the convex hull of $X$. The unit Euclidean sphere in $\R^m$ is $S^{m-1}=\{x \in \R^m : \norm{2}{x}=1\}$. 
We recall some learning-theoretic facts. Let $\Hyp$ be an arbitrary collection of subsets of $\scX$ (i.e., a concept class). The intersection class of $\Hyp$ is $I(\Hyp)=\bigcup_{i \in \Nat} \{h_1 \cap \ldots \cap h_i : h_1,\ldots,h_i \in \Hyp\}$. Given any $S \subset \scX$ and any $S' \subseteq S$ realized by some $h^{\star} \in \Hyp$, the smallest concept in $I(\Hyp)$ consistent with $S'$ is defined as $h^{\circ} = \bigcap\{h \in \Hyp : h \cap S = S'\}$. Note that $h^{\circ} \subseteq h^{\star}$. Finally, we recall the definition of learning with one-sided error:
\begin{definition}[\citet{Kivinen95}, Definition 4.4]\label{def:oneside}
An algorithm $\scA$ learns $\Hyp$ with one-sided error $\epsilon$ and confidence $\delta$ with $r$ examples if, for any target concept $h^{\star} \in \Hyp$ and any probability measure $\scP$ over $\scX$, by drawing $r$ independent labeled examples from $\scP$, the algorithm outputs a concept $h \subseteq h^{\star}$ such that $\scP(h^{\star} \setminus h) \le \epsilon$ with probability at least $1-\delta$.
\end{definition}

\section{Margin-based exact recovery of clusters in Euclidean spaces}
\label{sec:euclidean}
In this section, we consider the Euclidean setting $\scX=\R^m$. We show that the ellipsoidal margin assumption of~\citet{BCLP20} can be significantly generalized, while retaining the $\scO(\log n)$ query complexity, by introducing what we call the \emph{convex hull margin}. In a nutshell the convex hull margin says that, given any cluster $C$, any point not in $C$ is separated by the convex hull of $C$ by a distance at least $\gamma$ times the diameter of $C$. Instead of using the Euclidean metric, however, we allow distances to be measured by \emph{any} pseudometric over $\R^m$, which we do not need to know, and which may even differ from cluster to cluster. The only requirement is that the pseudometric be homogeneous and invariant under translation (i.e., induced by a seminorm). 
\begin{definition}[Convex hull margin]\label{def:ch_margin}
Let $D$ be the family of all pseudometrics induced by the seminorms over $\R^m$, and let $X \subset \R^m$ be a finite set. A clustering $\C=(C_1,\ldots,C_k)$ of $X$ has convex hull margin $\gamma$ if for every $i \in [k]$ there exists $d_i \in D$ such that:
\begin{align}
	d_i\big(X \setminus C_i, \conv(C_i)\big) > \gamma\, \diam_{d_i}(C_i)
\end{align} 
\end{definition}
This new definition has a few interesting properties. First, it strictly generalizes the ellipsoidal margin of~\citet{BCLP20} and the spherical margin of~\citet{ashtiani2016clustering}. To see this, let $D$ be the class of all pseudometrics over $\R^m$ that can be written as $d_{W}(x,y)=\dotp{x-y, W(x-y)}$ for some positive semidefinite matrix $W \in \R^{m \times m}$ (for the spherical margin, take $W=rI$ where $I$ is the identity matrix). Second, it strictly generalizes the classic SVM margin, which prescribes a distance of $\gamma \, \diam(X)$ between the clusters where $\diam(X)$ is the Euclidean diameter of $X$. Indeed, since every cluster has Euclidean diameter at most $\diam(X)$, a SVM margin of $\gamma$ implies a convex hull margin of $\gamma$. On the other hand, there are cases with arbitrarily small SVM margin but arbitrarily large convex hull margin, as we see in Section~\ref{sec:margin}. Third, the use of pseudometrics allows us to capture the Euclidean distance between the points after projection on a subspace $V \subset \R^m$, modelling scenarios where each cluster only ``cares'' about a certain subset of the features. 

Under the convex hull margin, we give a polynomial-time algorithm, named \CoolAlgo\ (for Convex Hull ExpAnsion Trick Recovery) that recovers $\C$ using $\scO(\log n)$ queries.
\begin{restatable}{retheorem}{cooltheorem}
\label{thm:cool}
Let $(X,O)$ be an instance whose latent clustering $\C$ has convex hull margin $\gamma > 0$. Then \CoolAlgo$(X,O,\gamma)$ deterministically outputs $\C$, runs in time $\poly(k,n,m)$, and with high probability makes a number of queries to $O$ bounded by $\scO\!\left(k^2 m^5 \left(1+\nicefrac{1}{\gamma}\right)^{m} \log\!\left(1+\nicefrac{1}{\gamma}\right) \log n\right)$.
\end{restatable}
To put this result in perspective, consider the algorithm of~\citet{BCLP20}. Under an ellipsoidal margin of $\gamma_{EL}$, that algorithm achieves a query bound of roughly $\big(\frac{m}{\gamma_{EL}})^m \log n$. One can check that their ellipsoidal margin of $\gamma_{EL}$ implies\footnote{Their definition of margin uses squared distances, so if a cluster has margin $\gamma_{EL}$ by their definition, then it has convex hull margin $\sqrt{1+\gamma_{EL}}-1$. This is why we need to make this apparently misplaced observation.} convex hull margin $\gamma \ge \frac{\gamma_{EL}}{3}$ for all $\gamma_{EL} \le 1$. Hence, in this range, our dependence on $\gamma$ is better by $\Theta(m^m)$ factors.

As we said, \CoolAlgo\ is based on boosting learners with one-sided error. What makes \CoolAlgo\ different, however, is the technique used to learn with one-sided error. To explain this, let us recall how the algorithm of~\cite{BCLP20} works. Their idea is to learn an approximate ellipsoidal separator that contains a constant fraction of some cluster $C_i$, and then remove from it all points that do not belong to $C_i$. Crucially, to learn the approximate separator, the algorithm needs a number of queries proportional to the VC-dimension of the corresponding class, which for $m$-dimensional ellipsoids is $\scO(m^2)$. Unfortunately, this approach does not work with the convex hull margin, since the class of allowed separators has unbounded VC-dimension (it is the class of all polytopes in $\R^m$).

To bypass this obstacle we develop a novel technique for learning with one-sided error that we call \emph{convex hull expansion trick}. This technique essentially amounts to the following procedure: draw a large labeled sample from $X$, take all the sampled points that belong to the same cluster $C$, and inflate their convex hull by a factor $1+\gamma$. This can be seen as a way to exploit the convex hull margin directly, without going through the VC-dimension of the separators. Such a trick may appear natural, but proving it to work is not trivial and requires a combination of results from probability, convex geometry, and PAC learning. In the rest of the section, we describe the trick and sketch the proof of Theorem~\ref{thm:cool}. For the complete proof, see the appendix.

\subsection{\CoolAlgo\ and the convex hull expansion trick}
The starting point of \CoolAlgo\ is the following idea. Let $s$ be a parameter to be set later. In each round, \CoolAlgo\ draws from $X$ a uniform random sample $S$ of size $|S|=ks$. Then, using the oracle, it determines the partition $S_1,\ldots,S_k$ induced by $\C$. As there are at most $k$ clusters, for at least one of them (say cluster $C$) we have $|S_C| \ge s$, a fact that allow us to use PAC bounds later on. Now, let $K=\conv(S_C)$, and let $d \in D$ be the pseudometric that gives the margin of $C$. Since $d$ is homogeneous and invariant under translation (recall that $D$ contains pseudometrics induced by seminorms), we know that any point $y$ such that $d(y,K) \le \gamma \diam_d(K)$ must belong to $C$ as well. Therefore, if we pick any point $z \in K$ and compute the scaling $Q=(1+\gamma)K$ with respect to $z$, $Q$ will not intersect clusters other than $C$. Hence, we can safely label all points in $X \cap Q$ as belonging to $C$. As long as we stick with this, we will not make mistakes --- in other words, we learn with one-sided error. However, this is not sufficient: in order to make actual progress, we must guarantee that $X \cap Q$ contains a good fraction of $C$. It is not obvious why this should be true: in the worst case, $X \cap Q$ could simply coincide with $S_C$ and we would have learned nothing.

It is here that the convex hull expansion trick enters the game. The key idea is to choose $z=\mu_K$, the center of mass of $K$. Since however computing $\mu_K$ is hard~\citep{Rademacher07Centroid}, we use an approximation. We give here the technical statement and a sketch of the proof (full proof in the appendix). The uniform probability measure $\scU$ over $K$ is defined by $\scU(K')=\frac{\vol(K')}{\vol(K)}$ for all measurable $K' \subseteq K$. A probability measure $\scP$ is $\epsilon$-uniform if $|\scP(K')-\scU(K')|\le \epsilon$ for all measurable $K' \subseteq K$.
\begin{restatable}[Convex hull expansion trick]{relemma}{tricklemma}
\label{lem:expanded_hull}
Fix $\gamma > 0$, and let $s = \Theta\big( m^5 \big(1+\nicefrac{1}{\gamma}\big)^{m} \log \big(1+ \nicefrac{1}{\gamma}\big)\big)$ large enough. Let $S_C$ be a sample of $s$ independent uniform random points from some cluster $C$, and let $K=\conv(S_C)$. Let $X_1,\ldots,X_N$ be independent random points sampled $\epsilon$-uniformly from $K$, with $\epsilon \in \Theta(m^{-1})$ small enough and $N \in \Theta(m^2)$ large enough, and let $z = \frac{1}{N}\sum_{i=1}^N X_i$. Finally, let $Q=(1+\gamma)K$ where the center of the scaling is $z$. Then, with probability arbitrarily close to $1$, we have $|Q \cap C| \ge \frac{1}{2}|C|$.
\end{restatable}
\begin{proof}[Sketch of the proof.]
To begin, suppose that $Q$ was obtained by scaling $K$ about its own center of mass $\mu_K$. Then, by a probabilistic argument, a result of~\citet{Naszodi2018} implies the existence of a polytope $P$ on roughly $(1+\nicefrac{1}{\gamma})^m$ vertices such that $K \subseteq P \subseteq Q$. Thus, if we can show that $X \cap P$ contains a good fraction of $C$, then $X \cap Q$ will contain a good fraction of $C$, too. To ensure that $X \cap P$ contains a good fraction of $C$, we observe that the class $\scP_t$ of all polytopes on at most $t$ vertices has VC-dimension at most $8 m^2 t \log_2 t$, by a recent result~\citep{Kupavskii2004-polyVCdim}. Hence, if we let $|S_C| = s \simeq (1+\nicefrac{1}{\gamma})^m $, then by standard PAC bounds $X \cap P$ contains half of $C$ with probability arbitrarily close to $1$. This holds because, as $K \subseteq P \subseteq Q$, then $P$ is consistent with $S_C$.

Now suppose that, in place of $\mu_K$, we use $z = \frac{1}{N}\sum_{i=1}^N X_i$. If $K$ is in isotropic position, then its radius is bounded by $m$, and therefore $\norm{2}{X_i} \le m$. This implies that, if $N=\Theta(m^2)$ and each $X_i$ comes from a distribution that is $\Theta(1/m)$-uniform over $K$, then with high probability $z$ is at distance $\eta=O(1)$ from $\mu_K$. In particular, by increasing $N$ by constant factors, we can make  arbitrarily small the probability that $\eta \le \nicefrac{1}{3}$. Then, by adapting the result of~\citet{Naszodi2018} with an extension of Grunbaum's inequality for convex bodies due to~\citet{BV04-ConvexRW}, we can show that $|Q \cap C| \ge \frac{1}{2}|C|$ with probability arbitrarily large.
\end{proof}

We conclude with a note on how to implement \CoolAlgo\ in polynomial time. The central point is to avoid computing $K=\conv(S_C)$ explicitly as an intersection of halfspaces, since this could take time $|S_C|^{\Theta(m)}$. Thus, we proceed as follows. First, we transform $S_C$ so that $\conv(S_C)$ is in what is called a near-isotropic position. To this end, we compute John's ellipsoid for $S_C$, and apply the map that turns that ellipsoid into a ball, which takes polynomial time. Afterwards, we can draw $X_1,\ldots,X_N$ efficiently via the ``hit-and-run from a corner'' algorithm of~\citet{lovasz2006hit}, in time $\scO\big(\poly(m,|S_C|)\log\frac{m}{\epsilon}\big)$ per sample, including the time to solve a linear program to determine when the walk hits the boundary of $K$. Once we have $z$, we rescale $S_C$ about $z$ itself, and label all points in $(1+\gamma)\conv(S_C)$ as belonging to $C$. This amounts to solving for every $x \in X$ a feasibility problem, which takes polynomial time. Finally, observe that the polytope $P$ appearing in the proof is only for the analysis, and is never computed. For a complete discussion, see the full proof.

\section{A more abstract view: the one-versus-all margin}
\label{sec:margin}
In this section, we consider the case where $\scX$ is a generic space equipped with a set of pseudometrics. Like in Section~\ref{sec:euclidean}, we want to formulate a notion of margin between clusters. However, now we cannot express the margin in terms of the diameter of convex hulls (since $\scX$ need not be a vector space and may have no notion of convexity at all), and we cannot apply any ``expansion trick'' (since the pseudometrics may not be homogeneous and invariant under translation). Instead, we adopt what we call the \emph{one-versus-all margin}. We prove that clusterings with positive one-versus-all margin can be recovered again with $\scO(\log n)$ oracle queries. However, we do not provide running time bounds like those of Section~\ref{sec:euclidean}: our algorithm is essentially based on computing an ERM, which in general may take time superpolynomial in $n$.

Next, we introduce the one-versus-all margin. 
\begin{definition}[One-versus-all margin]\label{def:ova_margin}
Choose $k$ pseudometrics $d_1,\ldots,d_k$ over $\scX$. A clustering $\C=(C_1,\ldots,C_k)$ of a finite set $X \subset \scX$ has one-versus-all margin $\gamma$ with respect to $d_1,\ldots,d_k$ if for all $i \in [k]$ we have:
\begin{align}
	d_i(X \setminus C_i, C_i) > \gamma\, \diam_{d_i}(C_i)
\end{align} 
\end{definition}
\textbf{Remark.} Unlike the convex hull margin (Definition~\ref{def:ch_margin}), here the pseudometrics $d_1,\ldots,d_k$ are fixed in advance. This is not a weakness of the above definition, but rather a strength of the convex hull margin, which allows one to recover clusters without even knowing $d_1,\ldots,d_k$.

Note also that, just like the convex hull margin, the one-versus-all margin in $\R^m$ is strictly more general than the SVM margin. The fact that SVM margin implies one-versus-all margin follows by the observation in Section~\ref{sec:euclidean}. The fact that one-versus-all margin does not imply SVM margin is shown in Lemma~\ref{lem:svm} below.
\begin{restatable}{relemma}{svmlemma}
\label{lem:svm}
For any $u \in \R^2$ let $d_u(x,y)=|\!\dotp{u,x-y}\!|$. For any $\eta > 0$ there exists a clustering $\C=(C_1,C_2)$ on a set $X \subset \R^2$ that has arbitrarily large one-versus-all margin with respect to $d_{(0,1)},d_{(1,0)}$, and yet $d_u(C_1,C_2) \le \eta\, \diam_{d_u}(X)$ for any $u \in \R^2 \setminus \orig$.
\end{restatable}
\begin{minipage}{.55\textwidth}
\begin{proof}[Proof sketch.]
Consider Figure~\ref{fig:SVMmargin}. The two points along the x axis belong to $C_1$, and the two points along the y axis belong to $C_2$. The pseudometric $d_1=d_{(0,1)}$ measures the distance along the y axis, and the pseudometric $d_2=d_{(1,0)}$ measures the distance along the x axis. It is easy to see that $d_1(C_1,C_2)>0=\diam_{d_1}(C_1)$, and that $d_2(C_1,C_2)>0=\diam_{d_2}(C_2)$. Hence, the one-versus-all margin of $\C$ with respect to $d_1,d_2$ is unbounded. Yet, for any $\eta > 0$ we can make $d_u(C_1,C_2) \le \eta\, \diam_{d_u}(X)$ for any nonzero vector $u \in \R^2$, by placing the endpoints of the clusters arbitrarily near the origin.
\end{proof}
\end{minipage}
\begin{minipage}{.04\textwidth}
\hfill
\end{minipage}
\begin{minipage}{.4\textwidth}
\centering
\begin{tikzpicture}[scale=1.2,every node/.style={circle,inner sep=0pt,minimum size=4pt,draw,fill}]
\def\etax{.2}
\def\etay{.16}
\draw[->,opacity=.7] (-.2,0) -- (13*\etax,0);
\draw[->,opacity=.7] (0,-.2) -- (0,12*\etay);
\node (p1) at (\etax,0) {};
\node (q1) at (10*\etax,0) {};
\node[fill=white] (p2) at (0,\etay) {};
\node[fill=white] (q2) at (0,10*\etay) {};
\end{tikzpicture}

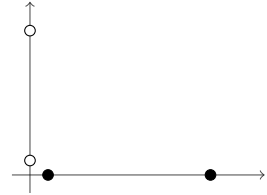
\captionof{figure}{An instance with arbitrarily small SVM margin but unbounded one-versus-all margin.}
\label{fig:SVMmargin}
\end{minipage}

\subsection{The one-versus-all-margin captures the stability of center-based clusterings}
Fix any pseudometric $d$ over $\scX$. A clustering $\C=(C_1,\ldots,C_k)$ of $X$ is \emph{center-based} if there exist $k$ points $c_1,\ldots,c_k \in \scX$, called \emph{centers}, such that for every $i \in [k]$ and every $x \in C_i$ we have $d(x,c_j) > d(x,c_i)$ for all $j \ne i$. In other terms, every point is assigned to the nearest center. It is well known that many popular center-based clustering problems, such as $k$-means or $k$-centers, are NP-hard to solve in general. However, those problems become polynomial-time solvable (or approximable) if the solution $\C$ meets certain stability properties. Here we show that two of these properties, the $\alpha$-center proximity of~\citet{Awashti2012-stability} and the $(1+\epsilon)$-perturbation resilience of~\citet{Bilu2012-stable}, imply a positive one-versus-all margin. Let us recall these properties. We define a $(1+\epsilon)$-perturbation of $d$ as any function $d'$ (which need not be a pseudometric) such that $d \le d' \le (1+\epsilon)d$.
\begin{definition}
Let $\C$ be a center-based clustering.
\begin{itemize}[nosep,leftmargin=12pt]
    \item $\C$ satisfies $\alpha$-center proximity with $\alpha > 1$ if, for all $i \in [k]$, for all $x \in C_i$ and all $j \ne i$ we have $d(x, c_j ) > \alpha\, d(x, c_i)$.
    \item $\C$ is $(1+\epsilon)$-perturbation resilient with $\epsilon > 0$ if it is induced by the same centers $c_1,\ldots,c_k$ under any $(1+\epsilon)$-perturbation of $d$.
\end{itemize}
\end{definition}
It is known that $(1+\epsilon)$-perturbation resilience implies $\alpha$-center proximity with $\alpha=1+\epsilon$, see~\citep{Awashti2012-stability}. Our result is:
\begin{restatable}{retheorem}{proxitheorem}
\label{thm:margin_from_proximity}
If $\C$ satisfies $\alpha$-center proximity, then it has one-versus-all margin at least $\frac{(\alpha-1)^2}{2(\alpha+1)}$. Hence, if $\C$ satisfies $(1+\epsilon)$-perturbation stability, then it has one-versus-all margin at least $\frac{\epsilon^2}{2(\epsilon+2)}$.
\end{restatable}
The proof of Theorem~\ref{thm:margin_from_proximity} is found in the appendix.

\subsection{Cluster recovery with one-versus-all margin}
\label{sub:margin_learning}
We conclude by showing that, if $\C$ has one-versus-all margin $\gamma > 0$, then one can recover $\C$ with $\scO(\log n)$ queries. Our algorithm, \MarginAlgo, considers the set of all possible clusters that satisfy the margin, and for each label selects the smallest hypotheses consistent with the sampled points. To prove that this approach works, we establish a formal connection between one-versus-all margin and one-sided-error learnability of the concept classes induced by all possible clusters with margin.

We need some further notation. As usual, let $d$ be any pseudometric over $\scX$. For any $X \subset \scX$ and any $r > 0$, we denote by $\PackNum(X,r,d)$ the maximum cardinality of any $A \subseteq X$ such that $d(x,y) > r$ for all distinct $x,y \in A$. From now on we assume that $\PackNum(X,r,d)$ is bounded, and for any $\gamma>0$, we define $M(\gamma,d) = \max\{\PackNum\big(B(x,r),\gamma r,d \big) : x \in \scX, r > 0\}$. Hence, for any $\gamma > 0$, any ball $B$ in $\scX$ contains at most $M(\gamma,d)$ points at pairwise distance greater than $\gamma$ times the radius of $B$, and some $B$ attains this bound.\footnote{If $\PackNum(X,r,d)$ is not bounded, then our results can be extended in the natural way, that is, we can prove a lower bound of $\Omega(n)$ queries for instances of $n$ points.} Finally, by $\vcdim(H,X)$ we denote the VC-dimension of a generic concept class $H$ over a set $X$.
\begin{restatable}[One-versus-all margin implies one-sided-error learnability]{relemma}{Hlemma}\label{lem:H_properties}
Let $d$ be any pseudometric over $\scX$. For any finite $X \subset \scX$ and any $\gamma > 0$, define the effective concept class over $X$:
\begin{align}
   H = \left\{C \subseteq X \,:\, d(X \setminus C, C) > \gamma \, \diam_d(C) \right\}
\end{align}
Then $H = I(H)$, and $\vcdim(H,X) \le M$ where $M = \max(2,M(\gamma,d))$. As a consequence, $H$ can be learned with one-sided error $\epsilon$ and confidence $\delta$ with $\scO\big(\epsilon^{-2}(M \log \nicefrac{1}{\epsilon} + \nicefrac{1}{\delta})\big)$ examples by choosing the smallest consistent hypothesis in $H$.
\end{restatable}
\begin{proof}[Sketch of the proof.]
To prove that $H = I(H)$, one can take any two $C_1,C_2 \in H$ and show that $C_1 \cap C_2$ satisfies the margin condition, too. To prove that $\vcdim(H,X) \le M$, we have two steps. First, let $\sldim(H,X)$ be the \emph{slicing dimension} of $H$. This is the size of the largest subset $S \subseteq X$ \emph{sliced} by $H$, i.e., such that for every $x \in S$ there is $C \in H$ giving $S \setminus x = S \cap C$, see~\citep{Kivinen95}. As the same work shows, we have $\vcdim(I(H),X) \le \sldim(H,X)$; hence, to prove $\vcdim(H,X) \le M$ it suffices to prove that $\sldim(H,X) \le M$. To this end, we use a packing argument. Suppose that $S \subseteq X$ is sliced by $H$, choose any $x \in S$, and let $C \in H$ such that $S \setminus x = S \cap C$. By construction of $H$, we know that $d(C,x) > \gamma \diam_d (C)$. Since $S \setminus x \subseteq C$, this yields:
\begin{align}
    d\big(S \setminus x, x\big) \ge d(C, x) > \gamma \, \diam_{d}(C) \ge \gamma \, \diam_{d}(S \setminus x)
\end{align}
It can be shown that this implies $d(S \setminus x, x) > \gamma \diam_{d}(S)$ for all $x \in S$, and, in turn, $|S| \le M$. The claim on the learnability with one-sided error holds by choice of the smallest consistent hypothesis in $H=I(H)$, combined with standard PAC bounds.
\end{proof}

We can now present our main result. We let $M(\gamma) = \max_{d \in \{d_1,\ldots,d_k\}} M(\gamma,d)$, with $d_1,\ldots,d_k$ as in Definition~\ref{def:ova_margin}.
\begin{restatable}{retheorem}{ovatheorem}
\label{thm:ova_margin}
Let $(X,O)$ be any instance whose latent clustering $\C$ has one-versus-all margin $\gamma > 0$ with respect to $d_1,\ldots,d_k$. Then \MarginAlgo$(X,O,\gamma)$ deterministically outputs $\C$ while making, with high probability, at most $\scO(M k \log k \log n)$ queries to $O$, where $M=\max(2,M(\gamma))$. Moreover, for any algorithm $\scA$ and for any $\gamma > 0$, there are instances with one-versus-all margin $\gamma$ on which $\scA$ makes $\Omega(M(2\gamma))$ queries in expectation.
\end{restatable}
\begin{proof}[Sketch of the proof.]
For the lower bounds, we take a set $X$ on $M(2\gamma)$ points at pairwise distance larger than $2\gamma$ times the radius of $X$, which is at least $\gamma$ times the diameter of $X$, and we draw a random clustering $\C$ in the form $(x, X \setminus x)$. One can see that $\C$ has one-versus-all margin $\gamma$, and simple arguments, coupled with Yao's principle for Monte Carlo algorithms, show that any algorithm needs $\Omega(M(2\gamma))$ queries in the worst case to return $\C$.
For the upper bounds, we show how to learn an expected constant fraction of $X$ with one-sided error using $\Theta(M k \lg k)$ queries; the rest follows by our general boosting argument. To begin, for each $i \in [k]$ we let $H_i = \left\{C \subseteq X \,:\, d_i(C, X \setminus C) > \gamma \, \diam_{d_i}(C) \right\}$. We then set $\epsilon=\nicefrac{1}{2k}$ and $\delta=\nicefrac{1}{2}$, and draw a labeled sample $S$ of size $\Theta(\epsilon^{-1}(M \log \nicefrac{1}{\epsilon} + \log \nicefrac{1}{\delta})) = \Theta(M k \log k)$. Finally, for each $i \in [k]$ we choose the smallest hypothesis $\hat{C}_i \in H_i$ consistent with the subset $S_i \subseteq S$ labeled as $i$, and we assign label $i$ to all points in $\hat{C}_i$. By Lemma~\ref{lem:H_properties}, with probability at least $\nicefrac{1}{2}$ we have $|\hat{C}_i| \ge |C_i| - \epsilon |X| = |C_i| - |X|/2k$. As $|\hat{C}_i| \ge 0$, this implies $\E|\hat{C}_i| \ge |C_i|/2 - |X|/4k$. By summing over all $i$, this shows that we are labelling correctly at least $|X|/4$ points in expectation.
\end{proof}

\textbf{Remark.} By Theorem~\ref{thm:ova_margin}, in $\R^m$ \MarginAlgo\ yields a $\scO(\log n)$ query bound even when the clusters are \emph{not} convex. However, this does not mean that \MarginAlgo\ subsumes \CoolAlgo. First, as noted above, here $d_1,\ldots,d_k$ are known in advance. Second, \MarginAlgo\ works by computing the smallest hypothesis $\hat{C}_i$ consistent with $S_i$ (see the proof of Theorem~\ref{thm:ova_margin}), which in general may take superpolynomial time. Indeed,  \CoolAlgo\ runs in polynomial time by \emph{not} computing $\hat{C}_i$ at all.

\section{One-versus-all clusterings}
In this section, we study active cluster recovery when the clusters can be realized by binary concepts from some concept class $\Hyp$. In other words, this is the clustering equivalent of the classic one-versus-all formulation of multiclass classifiers (see, e.g.,~\citep{SS14understanding}). We show that the active recoverability of clusterings in this setting is driven by the \emph{coslicing dimension} of $\Hyp$, a combinatorial quantity similar to the star number of~\cite{hanneke2015minimax} and the slicing dimension of~\citep{Kivinen95}. We also show that, for many natural concept classes in $\R^m$, a bounded coslicing dimension is equivalent to a positive one-versus-all margin.

As usual, let $\scX$ be any domain, $X \subset \scX$ be any finite set, and $\C=(C_1,\ldots,C_k)$ be a clustering of $X$. Let $\Hyp$ be any concept class over $\scX$. We say that $\C$ is realized by $\Hyp$ if for all $i \in [k]$ there is some $h_i \in \Hyp$ such that $C_i = X \cap h_i$. For example, the ellipsoidal clusters of~\citet{BCLP20} can be formulated by letting $\scX=\R^m$ and letting $\Hyp$ to be the family of all ellipsoids in $\R^m$, and the convex clusters of Section~\ref{sec:euclidean} can be formulated by letting $\scX=\R^m$ and letting $\Hyp$ to be the family of all polytopes in $\R^m$. Clearly, we expect that the number of active queries needed to recover $\C$ depends on the complexity of $\Hyp$. This leads us to the following question: what can we say about the recoverability of $\C$ in terms of the concept class $\Hyp$?

\subsection{A general characterization: the coslicing dimension}
In this section, we characterize the active recoverability of $\C$ via the \emph{coslicing dimension} of $\Hyp$, a combinatorial quantity similar to the slicing dimension of~\citet{Kivinen95} and the star number of~\citet{hanneke2015minimax}.
\begin{definition}
\label{def:cosldim}
We say that $\Hyp$ coslices $X \subseteq \scX$ if for all $x \in X$ there exist two concepts $h_x^+,h_x^- \in \Hyp$ such that $X \cap h_x^+ = \{x\}$ and $X \cap h_x^- = X \setminus \{x\}$. The coslicing dimension of $\Hyp$ is:
\begin{align*}
\cosl(\Hyp)=\sup\{|X| : X \text{ is cosliced by } \Hyp\}
\end{align*}
If $\Hyp$ coslices arbitrarily large sets then we let $\cosl(\Hyp)=\infty$.
\end{definition}
For instance, let $\scX=\R^m$. If $\Hyp$ is the class of linear separators, then $\cosl(\Hyp)=\infty$ (take $X$ to be the set of vertices of an $n$-vertex polytope and use the hyperplane separator theorem). If instead $\Hyp$ is the class of axis-aligned boxes, it can be shown that $\cosl(\Hyp)=2m$. Our main result is:
\begin{theorem}\label{thm:kclass_learning}
If $\cosl(\Hyp) < \infty$ then there is an algorithm that, given any $n$-point instance whose latent clustering $\C$ is realized by $\Hyp$, recovers $\C$ with $\scO\!\left(\cosl(\Hyp) \, k \log k \log n \right)$ queries with high probability. Moreover, for any algorithm $\scA$ there are instances on $\cosl(\Hyp)$ points whose latent clustering $\C$ is realized by $\Hyp$ where $\scA$ makes $\Omega(\cosl(\Hyp))$ queries in expectation to return $\C$. As a consequence, if $\cosl(\Hyp) = \infty$ then any algorithm needs $\Omega(n)$ queries in expectation to recover an $n$-point clustering realized by $\Hyp$.
\end{theorem}
\begin{proof}[Sketch of the proof.]
The proof follows the same ideas of the proof of Theorem~\ref{thm:ova_margin}. For the lower bounds, we take any $X$ cosliced by $\Hyp$ with $|X|=\cosl(\Hyp)$, and we draw a random clustering of $X$ in the form $(x, X \setminus x)$. For the upper bounds, for each $i \in [k]$ we define:
\begin{align}\label{eq:HiHi}
    H_i = \left\{C \,:\, C = C'_i \,\wedge\, (C_1',\ldots,C_k') \in P_k(X) \right\}
\end{align}
where $P_k(X)$ is the set of all clusterings of $X$ realized by $\Hyp$. As observed in the proof of Lemma~\ref{lem:H_properties}, we have the general relationship $\vcdim(I(H_i),X) \le \sldim(H_i,X)$ whenever $\sldim(H_i,X) < \infty$. Therefore, if we show that $\sldim(H_i,X) \le \cosl(\Hyp)$, by drawing a labeled sample of size $\Theta(\cosl(\Hyp) \, k \log k)$ we can recover the labels of an expected constant fraction of $X$, as in the proof of Lemma~\ref{lem:H_properties}. To prove that $\sldim(H_i,X) \le \cosl(\Hyp)$, let $U=\{x_1,\ldots,x_{\ell}\} \subseteq X$ be sliced by $H_i$. By construction of $H_i$, there are $\ell$ clusterings $\C_1,\ldots,\C_{\ell}$ realized by $\Hyp$ and such that $\C_j = (x_j, U \setminus x_j)$ for all $j \in [\ell]$. This implies that $U$ is cosliced by $\Hyp$. Hence, $|U| \le \cosl(\Hyp)$ and so $\sldim(H_i,X) \le \cosl(\Hyp)$.
\end{proof}
\textbf{Remark.} In the case of convex clusters in $\R^m$, we would let $\Hyp$ be the class of all convex polytopes, obtaining $\cosl(\Hyp)=\infty$ and thus Theorem~\ref{thm:kclass_learning} would not provide any useful bound. This is true even if $\C$ has convex hull margin $\gamma > 0$, although by Theorem~\ref{thm:cool} we know that $\C$ can be recovered with $\scO(\log n)$ queries. The same holds for the one-versus-all margin. This is because we defined the coslicing dimension as a function of $\Hyp$, whereas the margin depends on the instance $(X,O)$. To fix this, one can redefine the coslicing dimension in the form $\cosl(\Hyp,\scI)$ where $\scI$ is a class of instances, and adapt Theorem~\ref{thm:kclass_learning} correspondingly. Then, if every instance $(X,O) \in \scI$ has margin $\gamma > 0$, one can bound $\cosl(\Hyp,\scI)$ as a function of $\gamma$, retrieving the same type of bounds of Theorem~\ref{thm:cool} and Theorem~\ref{thm:ova_margin}.

\subsection{The one-versus-all margin, again!}
We look again at the Euclidean case, $\scX=\R^m$. Consider any concept class $\Hyp$. Theorem~\ref{thm:kclass_learning} and the above remark say that, if $\cosl(\Hyp,\scI)<\infty$, where $\scI$ is the class of allowed instances, then any clustering $\C$ realized by $\Hyp$ can be recovered with $\scO(\log n)$ queries, with no need for the one-versus-all margin (Definition~\ref{def:ova_margin}). We show that, for a wide family of concept classes in $\R^m$, both things are actually equivalent: we can achieve the $\scO(\log n)$ bound \emph{if and only} if the instances have positive one-versus-all margin. This establishes a connection between one-versus-all margin and active learnability of clusterings realized by concept classes in $\R^m$. In what follows, we assume that $\Hyp$ satisfies:
\begin{definition}
A concept class $\Hyp$ in ${\R^m}$ is non-fractal if there is $h \in \Hyp$ such that both $h$ and its complement contain a ball of positive radius. 
\end{definition}
This assumption avoids pathological cases (for instance, $\Hyp$ containing only hypotheses that are affine transformations of Cantor's set). Our result is:
\begin{restatable}{retheorem}{affinetheorem}
\label{thm:affine}
Let $\Hyp$ be a concept class in $\R^m$ that is non-fractal and closed under affine transformations. There is an algorithm that, given any instance whose latent clustering $\C$ has one-versus-all margin $\gamma$ and is realized by $\Hyp$, returns $\C$ while making $\scO(M k \log k \log n)$ queries with high probability, where $M=\max\big(2,(1+\nicefrac{4}{\gamma})^m\big)$. Moreover, for any algorithm $\scA$, there exist arbitrarily large $n$-point instances, whose latent clustering $\C$ has arbitrarily small one-versus-all margin and is realized by $\Hyp$, where $\scA$ makes $\Omega(n)$ queries in expectation to recover $\C$.
\end{restatable}
\begin{minipage}{.66\textwidth}
\begin{proof}[Sketch of the proof.]
The upper bounds follow by Theorem~\ref{thm:ova_margin} and the packing number of $\R^m$. For the lower bounds, we show that arbitrarily large packings of a sphere are cosliced by $\Hyp$. We use Figure~\ref{fig:packing_sketch} for reference. Let $h \in \R^m$ be such that both $h$ and its complement $\bar h$ contain a ball of positive radius. Then, for any $\rho > 0$ there exist a ball $B = B(c,r) \subseteq h$ with $r>0$, and a point $x \in \bar h$ such that $d(B,x) \le \rho$. Now take a sphere $S$ of radius $r' \ll r$ with center on the segment $xc$. Let $\eta = \sup_{y \in S \setminus B} d(x, y)$, and let $X$ be an $\eta$-packing of $S$, that is, a subset of points of $S$ such that $d(x',x'') > \eta$ for all distinct $x',x'' \in X$. Note how this implies that every $x' \in X \setminus x$ necessarily lies in $S \cap B$, and therefore, in $S \cap h$. Moreover, by letting $\eta/r' \to 0$, we can take $X$ arbitrarily large. Since $\Hyp$ is closed under affine transformations, by rotating $X$ it follows that for every $x \in X$ there is $h_x \in \Hyp$ such that $X \cap h_x = X \setminus x$. By applying the same argument to $\bar h$ and by complementation we can show that $X \cap h_x' = \{x\}$ for some $h'_x \in \Hyp$ for all $x \in X$ as well. Hence $\Hyp$ coslices arbitrarily large sets. To conclude, invoke Theorem~\ref{thm:kclass_learning}.
\end{proof}
\end{minipage}
\begin{minipage}{.03\textwidth}
    \hfill
\end{minipage}
\begin{minipage}{.3\textwidth}
    \begin{tikzpicture}[scale=.45,every node/.style={circle,inner sep=.8pt}]
    \node[draw] (orig) at (0,0) {};
    \node[above] (origtxt) at (orig) {$c$};
    \draw[draw=gray,fill=gray,fill opacity=.2] (0,0) circle (3.5);
    \node (B) at (3,-2.4) {$B$};
    \node[draw] (c) at (-2.6,0) {};
    \node[above] (cxt) at (c) {$c'$};
    \def\r{1}
    \def\N{9}
    \pgfmathsetmacro\ang{360/\N}
    \draw[densely dotted] (c) circle (\r);
    \foreach \i in {1,...,\N} {
        \node[draw,fill] (x) at ($(c)-(\ang*\i:\r)$) {};
    }
    \node[left] (xtx) at ($(x)$) {$x$};
    \node[right] (Stxt) at ($(x)+(1.6*\r,-\r)$) {$X$};
    \end{tikzpicture}
    
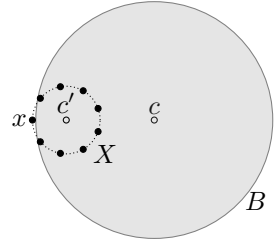
\captionof{figure}{The $\eta$-packing $X$ is in $B$, and thus in $h$, except for $x$ that lies in $\bar h$. By taking $B$ arbitrarily close to $x$, we can make $\eta$ arbitrarily small and thus $X$ arbitrarily large.}
    \label{fig:packing_sketch}
\end{minipage}

Note that Theorem~\ref{thm:affine} applies to several basic concept classes $\Hyp$. For instance, when $\Hyp$ is the class of all linear separators, the class of all ellipsoids, the class of all polytopes, and the class of all convex bodies (bounded or not, and possibly degenerate). It also includes more complex classes whose hypotheses are not convex: for instance, the class of all finite or infinite disjoint unions of balls, polytopes, or convex bodies.

\begin{ack}
The authors gratefully acknowledge partial support by the Google Focused Award ``Algorithms and Learning for AI'' (ALL4AI).  Nicolò Cesa-Bianchi is also supported by the MIUR PRIN grant Algorithms, Games, and Digital Markets (ALGADIMAR) and by the EU Horizon 2020 ICT-48 research and innovation action under grant agreement 951847, project ELISE (European Learning and Intelligent Systems Excellence).
\end{ack}

\bibliographystyle{abbrvnat}
\bibliography{references.bib}

\clearpage
\appendix
\section*{APPENDIX}

\section{Proof of Lemma~\ref{lem:expanded_hull}}
\label{apx:proof_expanded_hull}
\tricklemma*

\textbf{Preliminaries.}
Without loss of generality, we assume that $K$ has full rank (as one can always work in the subspace spanned by $S_C$, which can be computed in time $\scO(|S_C| m)$). For technical reasons, we let $\vartheta = \frac{1}{1+\gamma} \in (0,1)$ and prove the lemma for $s = \Omega\big( \frac{m^5}{\vartheta (1-\vartheta)^m} \ln \frac{1}{\vartheta (1-\vartheta)} \big)$ large enough. To see that any $s \in \Theta\big(m^5 \big(1+\nicefrac{1}{\gamma}\big)^{m} \ln \big(1+ \nicefrac{1}{\gamma}\big)\big)$ satisfies this assumption, first substitute $\vartheta$ to get:
\begin{align}
    \frac{1}{\vartheta (1-\vartheta)^m} \ln \frac{1}{\vartheta (1-\vartheta)}
    = 
    (1+\gamma)\left(\frac{1+\gamma}{\gamma}\right)^{m} \ln \frac{(1+\gamma)^2}{\gamma}
\end{align}
which is in $\scO\big((1+\nicefrac{1}{\gamma})^{m} (1+\gamma)  \ln (1+\nicefrac{1}{\gamma})\big)$. Now note that $(1+\gamma)\ln \big(1+ \nicefrac{1}{\gamma}\big)$ is bounded by $\scO\big(\gamma \cdot \nicefrac{1}{\gamma}\big)=\scO(1)$ for $\gamma > 1$, and by $2\ln\big(1+ \nicefrac{1}{\gamma}\big) = \scO\big(\ln\nicefrac{1}{\gamma}\big)$ when $\gamma \le 1$. Hence, in any case the term $(1+\gamma)  \ln (1+\nicefrac{1}{\gamma})$ is in $\scO(\ln (1+\nicefrac{1}{\gamma}))$. Therefore:
\begin{align}
    \big(1+\nicefrac{1}{\gamma}\big)^{m} \ln \big(1+ \nicefrac{1}{\gamma}\big) = \Omega\left( \frac{1}{\vartheta (1-\vartheta)^m} \ln \frac{1}{\vartheta (1-\vartheta)} \right) 
\end{align}
as claimed.

Before starting with the actual proof, we introduce some further definitions and notation.
\begin{definition}
A convex body $K \subset \R^m$ is in \emph{isotropic position}\footnote{Not to be confused with the definition of~\citep{NotesCB}, where the assumption $\int_K \dotp{x,u}^2 dx = 1$ is replaced by $\vol(K)=1$.}
 if it has center of mass in the origin, $\int_K x\, dx = 0$, and moment of inertia $1$ in every direction, $\int_K \dotp{x,u}^2 dx = 1$ for all $u \in S^{m-1}$.
\end{definition}
We define the norm $\norm{K}{\cdot} = \norm{2}{f(\cdot)}$ where $f=f_K$ is the unique affine transformation such that $f(K)$ is in isotropic position. We let $R(K) = \sup_{x \in K} \norm{2}{x}$ denote the Euclidean radius of $K$, and we let $R_K(K) = \sup_{x \in K} \norm{K}{x}$ denote the isotropic radius of $K$. We also let $d_K(x,y)=\norm{K}{x-y}$ be the isotropic distance of $K$.

Now, the proof has two steps. First, we show that the point $z= \frac{1}{N}\sum_{i=1}^N X_i$ is close to the centroid $\mu_K$ of $K$, according to $d_K(\cdot)$, with good probability. Second, we show that if this is the case, then $\frac{K}{\vartheta}$, where the scaling is meant about $z$, contains a polytope $P$ which contains $K$ and thus $S_C$, and which belongs to a class with VC dimension $\scO\big( \frac{m^5}{\vartheta (1-\vartheta)^m}\big)$. By standard PAC bounds this implies that $|P \cap C| \ge \frac{1}{2}|C|$, with a probability that can be made arbitrarily close to $1$ by adjusting the constants.

\subsubsection*{Step 1: $z$ is close to the centroid of $K$}
Let $\mu_K$ be the center of mass of $K$. We prove:
\begin{lemma}\label{lem:centroid_approx}
Fix any $\eta,p \ge 0$, and choose any $\epsilon \le \frac{\eta}{4(m+1)}$ and $N \ge \frac{16(m+1)^2}{p^2 \eta^2}$. If $X_1,\ldots,X_N$ are drawn independently and $\epsilon$-uniformly at random from $K$, and $\bar X = \frac{1}{N}X_i$, then:
\begin{align}
    \Pr\left(d_K\big(\bar X, \mu_K\big) \le \eta\right) \ge 1-p
\end{align}
\end{lemma}
For the proof of Lemma~\ref{lem:centroid_approx}, we need two ancillary results.
\begin{lemma}\label{lem:Rk}
$R_K(K) \le m+1$.
\end{lemma}
\begin{proof}
Consider $K$ in isotropic position, and let $K' = \vartheta K$ where $\vartheta = \vol(K)^{-1/m}$, so that $\vol(K')=1$. Then, $K'$ is in isotropic position according to the definition of \cite{NotesCB}. In this case, \citep[Theorem~1.2.4]{NotesCB} implies $R(K') \le (m+1) L_K$, where $L_K$ is the \emph{isotropic constant} which, for all $u \in S^{m-1}$, satisfies $\int_{K'} \dotp{ x,u}^2 d x = L_K^2$. Since $K' = \vartheta K$ and $\int_{K} \dotp{ x,u}^2 d  x = 1$ by the isotropy of $K$, we have $L_K = \vartheta$. Hence $R(K') \le (m+1) \vartheta$, that is, $R(K) \le m+1$.
\end{proof}
\begin{lemma}\label{lem:normEx}
If $X$ is drawn from an $\epsilon$-uniform distribution over $K$, then $\norm{K}{\E X} \le 2\epsilon(m+1)$.
\end{lemma}
\begin{proof}
Since $X$ is $\epsilon$-uniform over $K$, there exists a coupling $(X,Y)$ with $\Pr(X \ne Y) \le \epsilon$ and $Y$ uniform over $K$. Since $\norm{K}{\E Y}=0$, we have:
\begin{align}
 \norm{K}{\E X} = \norm{K}{\E[X - Y]} \le \Pr(X \ne Y) \!\sup_{ x, y \in K}\!\!\!d_K( x, y) \le \epsilon\, 2R_K(K) \le 2 \epsilon (m+1)
\end{align}
where the last inequality is given by Lemma~\ref{lem:Rk}.
\end{proof}
\begin{proof}[Proof of Lemma~\ref{lem:centroid_approx}]
For the sake of the analysis we look at $K$ from its isotropic position. Note that the $X_i$ are still $\epsilon$-uniform over $K$, since the affine map that places $K$ in isotropic position preserves volume ratios. The claim becomes:
\begin{align}
    \Pr\big(\norm{2}{\bar X} \le \eta\big) \ge 1-p
\end{align}
Now, $\norm{2}{\bar X} \le \norm{2}{\E \bar X} + \norm{2}{\bar X - \E \bar X}$. Thus, we show that $\norm{2}{\E \bar X} \le \frac{\eta}{2}$, and that $\norm{2}{\bar X - \E \bar X} \le \frac{\eta}{2}$ with probability at least $1-p$. For the first part, by Lemma~\ref{lem:normEx}, and since $\epsilon \le \frac{\eta}{4(m+1)}$, we obtain:
\begin{align}
    \norm{2}{\E \bar X} = \norm{2}{\E X_i} \le 2\epsilon(m+1) \le \frac{\eta}{2}
\end{align}
For the second part, by Lemma~\ref{lem:Rk} we have $\norm{2}{X_i} \le m+1$ for all $i$. Therefore, if we let $Y_i=X_i-\E X_i$ for all $i$, we have $\norm{2}{Y_i} \le 2(m+1)$ and thus $\norm{2}{Y_i}^2 \le 4(m+1)^2$. Now let $\bar Y = \frac{1}{N}\sum_{i=1}^N Y_i$. Since the $Y_i$ are independent and with $\E Y_i = 0$, then $\E \dotp{Y_i,Y_j}=0$ whenever $i \ne j$ and therefore:
\begin{align}
    \E\norm{2}{\bar Y}^2 =
    \E\bigg(\frac{1}{N^2}\sum_{i,j=1}^N\dotp{Y_i,Y_j}\bigg) =
    \frac{1}{N^2} \sum_{i=1}^N \E \norm{2}{Y_i}^2  \le \frac{4(m+1)^2}{N}
\end{align}
Plugging in our value $N \ge \frac{16(m+1)^2}{p^2 \eta^2}$, and using Jensen's inequality, we obtain:
\begin{align}
    \big(\E\norm{2}{\bar Y}\big)^2 \le \E\norm{2}{\bar Y}^2 \le \frac{p^2 \eta^2}{4}
\end{align}
Therefore $\E\norm{2}{\bar Y} \le \frac{p \eta}{2}$, which by Markov's inequality implies that $\Pr\big(\norm{2}{\bar Y} > \frac{\eta}{2}\big) < p$. By noting that $\bar Y = \bar X - \E \bar X$, the proof is complete.
\end{proof}

\subsubsection*{Step 2: showing a tight enclosing polytope}
We prove:
\begin{lemma}\label{lem:Pexists}
Let $z \in \R^m$ such that $d_K(z,\mu_K) \le \frac{1}{e}-\frac{1}{3}$. For any $\vartheta \in (0,1)$ there exists a polytope $P$ on $t = \scO\big( \frac{m^2}{\vartheta (1-\vartheta)^m}\big)$ vertices such that $K \subseteq P \subseteq \frac{K}{\vartheta}$, where the scaling is about $z$.
\end{lemma}
For the proof, we need some ancillary results.
\begin{theorem}[\cite{BV04-ConvexRW}, Theorem 3]
\label{thm:Grunbaum_general}
Let $K$ be a convex set in isotropic position and $z$ be a point at distance $t$ from its centroid. Then any halfspace containing $z$ contains at least $\frac{1}{e}-t$ of the volume of $K$.
\end{theorem}
Now we adapt a result by \citet{Naszodi2018}, recalled here for convenience. We say that a halfspace $F$ supports a convex body from outside if $F$ intersects the boundary of the body, but not its interior.
\begin{lemma}[Lemma 2.1, \cite{Naszodi2018}]
Let $0 < \vartheta < 1$, and $F$ be a halfspace that supports $\vartheta K$ from outside, where the scaling is about $\mu_K$. Then:
\begin{align}
    \vol(K \cap F) \ge \vol(K) \cdot (1-\vartheta)^m \frac{1}{e}
\end{align}
\end{lemma}
Our adaptation is:
\begin{lemma}\label{lem:cutvol}
Let $z \in \R^m$, let $0 < \vartheta < 1$, and let $F$ be a half-space that supports $\vartheta K$ from outside, where the scaling is about $z$. Then:
\begin{align}
    \vol(K \cap F) \ge \vol(K) \cdot (1-\vartheta)^m \left(\frac{1}{e} - d_K(\mu_K,z) \right)
\end{align}
\end{lemma}
\begin{proof}
The proof is similar to the proof of the original lemma. See Figure~\ref{fig:cutvol} for reference. Let $F_{0}$ be a translate of $F$ whose boundary contains $z$, and let $K_0 = K \cap F_0$. Let $F_1$ be a translate of $F$ that supports $K$ from outside, and let $p \in F_1 \cap K$. Now consider $K_0' = \vartheta p + (1-\vartheta)K_0$, that is, the homothetic copy of $K_0$ with center $p$ and ratio $1-\vartheta$. The crucial observation is that $F = \vartheta p + (1-\vartheta)F_0$, which implies $K_0' \subset K \cap F$. Clearly $\vol(K_{0}') = (1-\vartheta)^m \vol(K_0)$. Moreover, by Theorem~\ref{thm:Grunbaum_general} we have $\vol(K_0) \ge \vol(K)(\nicefrac{1}{e} - t)$ where $t=d_K(\mu_K,z)$; this holds because mapping $K$ to its isotropic position preserves volume ratios. This concludes the proof.
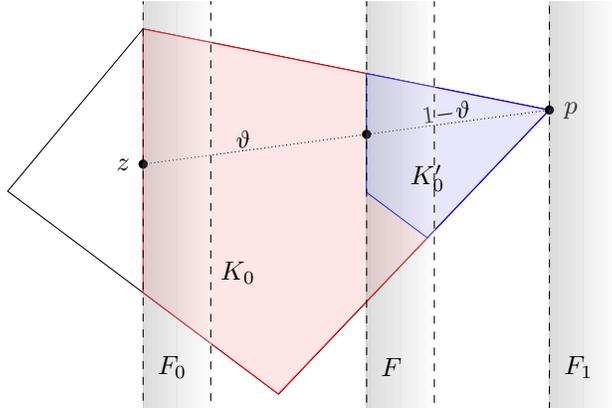
\begin{figure}[h!]
    \centering
    \begin{tikzpicture}[scale=1.8,every node/.style={circle,inner sep=0pt,minimum size=3.5}]
\tikzset{
  cdc/.style={every coordinate/.try},
  halfspace/.style={draw, dashed, left color=gray, right color=white, path fading=east, fill opacity=.2}
}
\def\thet{.55} 

\clip (-1.1,-1.8) rectangle (3.5,1.2);

\coordinate (orig) at (0,0);

\coordinate (a) at (-1,-.2);
\coordinate (b) at (0,1);
\coordinate (c) at (3,.4);
\coordinate (d) at (1,-1.7);

\coordinate (p) at (c); 
\coordinate (orig1) at ($(orig)!{\thet}!(p)$); 
\coordinate (adF0) at (intersection cs: first line={(a)--(d)}, second line={(orig)--($(orig)+(0,-3)$)});

\draw (a) -- (b) -- (c) -- (d) -- cycle;

\coordinate (a1) at ($\thet*(a)$);
\coordinate (b1) at ($\thet*(b)$);
\coordinate (c1) at ($\thet*(c)$);
\coordinate (d1) at ($\thet*(d)$);
\coordinate (adF01) at ($\thet*(adF0)$);
\begin{scope}[every coordinate/.style={scale={.999}}]
    \path[draw=red,fill=red!10] ([cdc]b) -- ([cdc]c) -- ([cdc]d) -- ([cdc]adF0) -- cycle;
\end{scope}

\coordinate (a2) at ($(a)!{\thet}!(p)$);
\coordinate (b2) at ($(b)!{\thet}!(p)$);
\coordinate (c2) at ($(c)!{\thet}!(p)$);
\coordinate (d2) at ($(d)!{\thet}!(p)$);
\coordinate (adF02) at ($(adF0)!{\thet}!(p)$);
\begin{scope}[every coordinate/.style={scale={(.999)}}]
    \path[draw=blue,fill=blue!10] ([cdc]b2) -- ([cdc]c2) -- ([cdc]d2) -- ([cdc]adF02) -- cycle;
\end{scope}

\node[fill] (z) at (orig) {};
\node[fill] (z1) at (orig1) {};
\draw[densely dotted] (orig) -- (p);
\path (orig) to node[above,sloped,near start]{\small$\vartheta$} (p);
\path (orig) to node[above=-6pt,sloped,near end]{\small$1\!-\!\vartheta$} (p);
\node[fill] (cnode) at (c) {};
\node[right=4pt] (ctxt) at (c) {$p$};
\node[left=4pt] (ztxt) at (z) {$z$};

\shade[halfspace] ($(orig)+(0,-5)$) rectangle ++(.5,10);
\node[right=4pt] (F0) at (orig |- 1,-1.5) {$F_0$};
\shade[halfspace] ($(orig1)+(0,-5)$) rectangle ++(.5,10);
\node[right=4pt] (F0) at (orig1 |- 1,-1.5) {$F$};
\shade[halfspace] ($(p)+(0,-5)$) rectangle ++(.5,10);
\node[right=4pt] (F0) at (p |- 1,-1.5) {$F_1$};

\node (K0) at (.7,-.8) {$K_0$};
\node (K01) at (2.1,-.1) {$K_0'$};

\end{tikzpicture}
    \caption{a visual proof of Lemma~\ref{lem:cutvol}, with $d(z,p)=1$ for simplicity.}
    \label{fig:cutvol}
\end{figure}
\end{proof}
\begin{proof}[Proof of Lemma~\ref{lem:Pexists}]
We adapt the construction behind \citep[Theorem 1.2]{Naszodi2018}, by replacing $\varepsilon = \frac{(1-\vartheta)^m}{e}$ with $\varepsilon = \frac{(1-\vartheta)^m}{3}$.
The theorem then says that, if we set:
\begin{align}
    t = \left\lceil C \frac{(m+1)3}{(1-\vartheta)^m}\ln \frac{3}{(1-\vartheta)^m} \right\rceil 
\end{align}
and we let $X_1,\ldots, X_t$ be $t$ points chosen independently and uniformly at random from $K$, and let $P=\conv(X_1,\ldots,X_t)$, then $\vartheta K \subseteq P \subseteq K$ with probability at least $1-p$, where
\begin{align}
    p = 4 \left(11 C^2 \left(\frac{(1-\vartheta)^m}{3}\right)^{\!\!C-2}\right)^{m+1}
\end{align}
Now, we choose $C = \Theta(\frac{1}{\vartheta}\ln \frac{1}{\vartheta})$ large enough.
On the one hand, we obtain:
\begin{align}
    t = \left\lceil C \frac{(m+1)3}{(1-\vartheta)^m}\ln \frac{3}{(1-\vartheta)^m} \right\rceil
    = \scO\left( \frac{m^2}{\vartheta (1-\vartheta)^m} \ln \frac{1}{\vartheta} \ln \frac{1}{1-\vartheta}  \right)
\end{align}
Since $\vartheta \in (0,1)$, we have $\ln \frac{1}{\vartheta} \ln \frac{1}{1-\vartheta}  = \ln(1+\nicefrac{1}{x})\ln(1+x)$ where $x=\frac{\vartheta}{1-\vartheta} > 0$. However, $\ln(1+\nicefrac{1}{x})\ln(1+x) < 1$ for all $x > 0$. Therefore, $t \in \scO\big( \frac{m^2}{\vartheta (1-\vartheta)^m}\big)$. On the other hand, by setting $C$ large enough we can make $C^2 \big(\frac{(1-\vartheta)^m}{3}\big)^{C-2}$ arbitrarily small, and therefore $p < 1$.

Since $p<1$, we conclude that \emph{there exists} a polytope $P$ on $t \in \scO\big( \frac{m^2}{\vartheta (1-\vartheta)^m}\big)$ vertices such that $\vartheta K \subseteq P \subseteq K$. To conclude, instead of $P$ simply take $\frac{P}{\vartheta}$ where the scaling is about $z$.
\end{proof}

\subsubsection*{Wrap-up}
First, by Lemma~\ref{lem:centroid_approx}, by taking $N \in \scO(m^2)$ large enough we can make $d_K(\mu_K,z) \le \frac{1}{e}-\frac{1}{3}$ with probability arbitrarily close to $1$. Now let $\scP_t$ be the family of all polytopes in $\R^m$ on at most $t$ vertices. For $t \in \scO\big( \frac{m^2}{\vartheta (1-\vartheta)^m}\big)$ large enough, Lemma~\ref{lem:Pexists} implies that there exists some $P \in \scP_t$ such that $K \subseteq P \subseteq \frac{K}{\vartheta}$.

Now we prove that, by choosing $s$ large enough, with probability arbitrarily close to $1$ we have $|\frac{K}{\vartheta} \cap C| \ge \frac{1}{2}|C|$. First, by a recent result of~\cite{Kupavskii2004-polyVCdim}, we have $\vcdim(\scP_t) \le 8 m^2 t \log_2 t$. For our $t$ this yields
\begin{align}
  \vcdim(\scP_t) = \scO\!\left( m^2 \frac{m^2}{\vartheta (1-\vartheta)^m} \ln \frac{m^2}{\vartheta (1-\vartheta)^m} \right) = \scO\!\left( \frac{m^5}{\vartheta (1-\vartheta)^m} \ln \frac{1}{\vartheta (1-\vartheta)} \right)
\end{align}
where in the last equality we used $\ln \frac{m^2}{\vartheta (1-\vartheta)^m} = \ln \frac{m^{(2/m)m}}{\vartheta (1-\vartheta)^m} \le m \ln \frac{m^{2/m}}{\vartheta (1-\vartheta)} = \scO\big(m \ln \frac{1}{\vartheta (1-\vartheta)}\big)$.
Hence, by choosing $|S_C| = s = \scO\big( \frac{m^5}{\vartheta (1-\vartheta)^m} \ln \frac{1}{\vartheta (1-\vartheta)} \big)$ large enough, for any constant $c,\epsilon,\delta$ we can make:
\begin{align}
    |S_C| \ge c\, \frac{\vcdim(\scP_t) \, \ln \frac{1}{\epsilon} + \ln \frac{1}{\delta}}{\epsilon}
\end{align}
Since $P$ is consistent with $S_C$, that is, $P \supset S_C$, then by standard PAC bounds we have $|P \cap C| \ge (1-\epsilon)|C|$ with probability at least $1-\delta$. But $P \subseteq \frac{K}{\vartheta}$, and therefore $|\frac{K}{\vartheta} \cap C| \ge (1-\epsilon)|C|$ with probability at least $1-\delta$. By adjusting the constants this yields the thesis of Lemma~\ref{lem:expanded_hull}.

\section{Proof of Theorem~\ref{thm:cool}}
\cooltheorem*
We give the pseudocode of the algorithm for reference. First, we prove the correctness and the query bound. Then we show the running time bound. Note that, for readability, the pseudocode given here is high-level; the actual implementation is more complex, see below.

\begin{algorithm}
\caption{
\label{alg:exphull}
\ExpandHull$(K,1+\gamma)$
}
\begin{algorithmic}
\State let $N = \Omega(m^2)$ large enough
\State let $\epsilon = \scO(m^{-1})$ small enough
\State draw $N$ i.i.d.\ random points $X_1,\ldots,X_N$ from any $\epsilon$-uniform distribution over $K$
\State let $z = \frac{1}{N}\sum_{i=1}^N X_i$
\State \Return $z + (1+\gamma)(K-z)$
\end{algorithmic}
\end{algorithm}

\begin{algorithm}
\caption{\CoolAlgo$(X,O,\gamma)$}
\begin{algorithmic}
\While{$X \ne \emptyset$}
\State let $s = \scO\big(m^5 \big(1+\nicefrac{1}{\gamma}\big)^{m} \ln \big(1+ \nicefrac{1}{\gamma}\big)\big)$ large enough
\State draw a uniform random sample $S$ of size $\min(|X|,ks)$ from $X$, without repetition
\State learn the labels of $S$ with $ks$ queries to $O$
\State let $S_i$ be the points of $S$ having label $i$
\For{$i=1,\ldots,k$}
\State $K = \conv(S_i)$
\State $Q = $ \ExpandHull$(K,1+\gamma)$
\State $\hat{C}_i = Q \cap X$
\State label all points of $\hat{C}_i$ with label $i$
\State $X = X \setminus \hat{C}_i$
\EndFor
\EndWhile
\end{algorithmic}
\end{algorithm}

\textbf{Correctness and query bound.} We prove that, at each round, for some $i$ we recover at least half of the points in $C_i$ with probability $1-\delta$, where $\delta$ can be made arbitrarily small by adjusting the constants. Let $S_i$ be the subset of the sample $S$ having label $i$. Since there are at most $k$ clusters and $|S| = k s$, for some $i$ we will have $|S_i| \ge s$. Now we apply Lemma~\ref{lem:expanded_hull} to $K=\conv(S_i)$. Since $s$ satisfies the hypotheses, the lemma says that $\hat{C}_i = Q \cap X$ has size $|\hat{C_i}| \ge \frac{|C_i|}{2}$ with probability arbitrarily close to $1$ (that is, with probability $1-\delta$ as above). It remains to show that $\hat{C_i} \subseteq C_i$. Let $d$ be any pseudometric that is homogeneous and invariant under translation. Then, any point $ y \in (1+\gamma)K \cap X$ satisfies $d( y, K) \le \gamma\, \diam_d(K)$. But $K = \conv(S_i) \subseteq \conv(C_i)$. Therefore $\diam_d(K) \le \diam_d(C_i)$. Hence $d( y,\conv(C_i)) \le \gamma\, \diam_d(C_i)$. By the convex margin assumption, this implies that $ y \in C_i$. This also proves the correctness of the algorithm. The total query bound follows as in Lemma 3 of~\citep{BCLP20}, whose algorithm also recovers an expected fraction $\frac{1}{4k}$ of all points in each round.

\textbf{Running time bound.}
First, we analyze the running time of \CoolAlgo\ excluding the call to \ExpandHull. Drawing and labeling the samples obviously cost $\scO(n)$ time throughout the entire execution. In the main loop, $K$ is actually not computed explicitly --- see below. Similarly, $Q$ is simply a set of points obtained by rescaling $S_i$ about some point in space. Thus, computing $\hat{C}_i = Q \cap X$ and labeling its points boils down to deciding, for all $x \in X$, if $x$ can be written as $\sum_{x_j \in Q_j} \lambda_j x_j$ for a set of coefficients $\lambda_1,\ldots,\lambda_{|Q|} \in [0,1]$. This can be done with polynomial precision using any polynomial-time solver for linear programs (say, the ellipsoid method). 

Let us now turn to \ExpandHull. The computationally expensive part is drawing $N$ points from an $\epsilon$-uniform distribution over $K=\conv(S_i)$. This can be done with any method of choice. Here, we consider the ``hit-and-run from a corner'' algorithm of~\citet{lovasz2006hit}, which implements a fast mixing random walk whose stationary distribution is uniform over any convex body. We remark that other methods for computing approximate centers exist, see for example~\citep{basu2017centerpoints}.

The implementation is as follows. First, we put $K$ in near-isotropic position by computing the minimum volume enclosing ellipsoid (MVEE) and then applying an affine transformation to make the MVEE into the ball of unit radius. As shown in~\citep{khachiyan1996rounding}, this operation takes time $|S_i|m^2\big(\ln m + \ln\ln|S_i|\big)$. After this transformation, let $\mu$ be the center of the unit ball that contains $K$. Observe that $\mu$ is at distance at least $\frac{1}{m}$ from the boundary of $K$: this holds since, by John's theorem, the ball of radius $\frac{1}{m}$ centered at $\mu$ is entirely contained in $K$. Now we execute the hit-and-run from a corner algorithm starting at $\mu$. By the results of~\citet{lovasz2006hit}, we have the following bound.
\begin{lemma}[See~\citet{lovasz2006hit}, Corollary~1.2]
Assume hit-and-run is started from $\mu$. For any $\epsilon > 0$, the distribution of the random walk after
\[
    t = \Theta\left(m^4 \ln\frac{m}{\epsilon}\right)
\]
steps is $\epsilon$-uniform over $K$.
\end{lemma}
It remains to implement the hit-and-run algorithm over $K$. To this end we need to solve the following problem: given a generic point $x \in K$ and a vector $u \in S^{n-1}$, determine the intersection of the ray $\{x + \alpha u\}_{\alpha \ge 0}$ with the boundary of $K$. This amounts to solving a linear program that searches for the maximum value $\alpha \ge 0$ such that $x + \alpha u$ can be written as $\sum_{x_j \in S_i} \lambda_j x_j$ for a set of coefficients $\lambda_1,\ldots,\lambda_{|S_i|} \in [0,1]$. We can solve such an LP in time $t_{K} = \poly(|S_i|,m)$ with polynomial precision using any polynomial-time solver for linear programs. The total time to draw the $N$ samples is therefore:
\begin{align}
  \scO\left(|S_i|m^2\big(\ln m + \ln\ln|S_i|\big) + N t_K\, m^4\ln\frac{m}{\epsilon}\right)
\end{align}
As we set $N=\scO(m^2)$ and $\epsilon=\scO(m^{-1})$, the total running time is:
\begin{align}
  \scO\left(|S_i|m^2\big(\ln m + \ln\ln|S_i|\big) + t_K\, m^6 \ln m\right)
\end{align}
which is in $\poly(|S_i|,m) = \poly(n,m)$.

\section{Proof of Lemma~\ref{lem:svm}}\label{apx:proof_lem_svm}
\svmlemma*
Let $u_1=(1,0)$ and $u_2=(0,1)$, and for some constant $a$ independent of $\eta$ and to be fixed later, consider the set $X$ consisting of the four points (see Figure~\ref{fig:SVMmargin}):
\begin{align}
p_1 &= \eta\, u_1, \quad q_1 = a\, u_1 \\
p_2 &= \eta\, u_2, \quad q_2 = a\, u_2
\end{align}
Finally, let $\C=(C_1,C_2)$ where $C_1=\{ p_1,q_1 \}$ and $C_2=\{ p_2,q_2\}$.

Consider the two pseudometrics $d_1=d_{(0,1)}$ and $d_2=d_{(1,0)}$. Then $\diam_{d_1}(C_1) = 0$ and $d_1(C_1,C_2)=\eta$, and vice versa, $\diam_{d_2}(C_2) = 0$ and $d_2(C_1,C_2)=\eta$. Thus, the one-versus-all margin of $\C$ with respect to $d_1,d_2$ is unbounded.

Now choose any $u \in \R^2 \setminus \orig$. Without loss of generality, by rescaling we can assume $u$ to be a unit vector. In this case, we have $d_u(C_1,C_2) \le d_u(p_1,p_2) \le \norm{2}{p_1-p_2} = \eta \sqrt{2}$. Yet, the convex hull of $X$ contains a ball of radius $\Omega(1)$, and therefore $\diam_d(X) = \Omega(1)$, where the constants depend on $a$. Hence, we can make $d(C_1,C_2) \le \eta\, \diam_d(X)$ by choosing $a$ appropriately.

\section{Proof of Theorem~\ref{thm:margin_from_proximity}}
\proxitheorem*
Consider any cluster $C_i$ with center $c_i$. We must show that any $y \in C_j$ with $j \ne i$ satisfies $d(y,x) > \frac{(\alpha-1)^2}{2(\alpha+1)} \diam_d(C_i)$ for all $x \in C_i$. Let $x' = \arg \max_{x \in C_i} d(x,c_i)$. Clearly, if $d(x',c_i)=0$ then all points of $C_i$ coincide and $\diam_d(C_i)=0$. If this is the case, then by the $\alpha$-center proximity, for any $x \in C_i$ and any $y \in C_j$ we have $d(y,x) = d(y,c_i) > \alpha\, d(y,c_j) \ge 0$. Therefore $d(y,x) > a \, \diam_d(C_i)$ for any $a > 0$, which in particular proves our thesis.

Suppose instead that $d(x',c_i) > 0$. Choose any $x \in C_i$ and any $y \in C_j$. Now, we have two cases.

\textbf{Case 1: $x=c_i$.} In this case we proceed as follows. Bear in mind that $d(x',c_i) = \diam_d(C_i)$, $d(y,c_i)=d(y,x)$, and $d(y,c_j) < \frac{1}{\alpha} d(y,c_i)$.
\begin{align}
\diam_d(C_i) &= d(x',c_i) \\
&< \frac{1}{\alpha} d(x',c_j) \\
& \le \frac{1}{\alpha} \Big( d(x',c_i) + d(c_i,y) + d(y,c_j)\Big) \\
& < \frac{1}{\alpha} \left( \diam_d(C_i) + d(y,x) +  \frac{1}{\alpha} d(y,c_i)\right) \\
& = \frac{1}{\alpha} \diam_d(C_i) + \frac{1}{\alpha} d(y,x) + \frac{1}{\alpha^2} d(y,x)
\end{align}
from which we infer
\begin{align}
d(y,x) > \frac{\alpha(\alpha-1)}{\alpha+1} \diam_d(C_i) > \frac{(\alpha-1)^2}{2(\alpha+1)} \diam_d(C_i)
\end{align}

\textbf{Case 2: $x \ne c_i$.} In this case, $d(x, c_i) > 0$, and we start by deriving:
\begin{align}
  d(y,x) &\ge d(x,c_j) - d(y,c_j)
  \\ &> d(x,c_j) - \frac{1}{\alpha}d(y,c_i)
  \\ &\ge d(x,c_j) - \frac{1}{\alpha}\big(d(y,x) + d(x,c_i)\big) 
  \\ &= -\frac{1}{\alpha}d(y,x) + d(x,c_j)  -\frac{1}{\alpha}d(x,c_i)
\end{align}
and thus
\begin{align}
  d(y,x)\left(\frac{\alpha+1}{\alpha}\right) &> d(x,c_j) -\frac{1}{\alpha}d(x,c_i)
\end{align}
which yields
\begin{align}
  d(y,x) &> \frac{\alpha}{\alpha+1} d(x,c_j) - \frac{1}{\alpha+1} d(x,c_i)\label{eq:dyx}
\end{align}

Let $\beta = \frac{d(x',c_i)}{d(x,c_i)}$. Observe that $\diam_d(C_i) \le 2 \beta d(x,c_i)$ and therefore $d(x,c_i)\ge \frac{\diam_d(C_i)}{2 \beta}$.

Now we consider two cases. First, suppose that $\beta \le \frac{\alpha+1}{\alpha-1}$. In this case, we apply $d(x,c_j)> \alpha d(x,c_i)$ to~\eqref{eq:dyx} to obtain:
\begin{align}
  d(y,x) &> \frac{\alpha^2}{\alpha+1} d(x,c_i) - \frac{1}{\alpha+1} d(x,c_i)
  \\&= d(x,c_i)(\alpha-1)
  \\&\ge \frac{\diam_d(C_i)}{2 \beta}(\alpha-1)
  \\&\ge  \diam_d(C_i)\frac{(\alpha-1)^2}{2(\alpha+1)}
\end{align}

Suppose instead that $\beta > \frac{\alpha+1}{\alpha-1}$. Since we chose $x' \in C_i$ such that $d(x',c_i)=\beta d(x,c_i)$, we obtain:
\begin{align}
  d(x,c_j) &> d(x',c_j) - d(x,x')
  \\ &> \alpha d(x',c_i) - \big(d(x,c_i)+d(x',c_i)\big) 
  \\ &= \alpha \beta d(x,c_i) - \big(d(x,c_i)+ \beta d(x,c_i)\big) 
  \\ &= d(x,c_i) ((\alpha-1)\beta- 1)
\end{align}
Combining this with~\eqref{eq:dyx}, we obtain:
\begin{align}
	d(y,x) &> d(x,c_i) \left(\frac{\alpha}{\alpha+1} ((\alpha-1)\beta- 1) - \frac{1}{\alpha+1}\right)
	\\ &= d(x,c_i) \left(\frac{\alpha(\alpha-1)\beta}{\alpha+1} - 1 \right)
	\\ &\ge \frac{\diam_d(C_i)}{2 \beta} \left(\frac{\alpha(\alpha-1)\beta}{\alpha+1} - 1 \right)
	\\ &= \diam_d(C_i) \left(\frac{\alpha(\alpha-1)}{2(\alpha+1)} - \frac{1}{2\beta} \right)
	\\ &> \diam_d(C_i) \left(\frac{\alpha(\alpha-1)-(\alpha-1)}{2(\alpha+1)} \right)
	\\ &= \diam_d(C_i) \frac{(\alpha-1)^2}{2(\alpha+1)}
\end{align}
Hence, in all cases, we obtain $d(y,x) > \diam_d(C_i) \frac{(\alpha-1)^2}{2(\alpha+1)}$. This concludes the proof.

\section{Proof of Lemma~\ref{lem:H_properties}}
\Hlemma*
For the first claim, we start by showing that $H=I(H)$. Let $C_1,C_2 \in H$. We show that $C := C_1 \cap C_2 \in H$, too. Consider any $y \in X \setminus C$, and without loss of generality assume that $y \notin C_1$. Since $C \subseteq C_1$, we have $\diam_d(C_1) \ge \diam_d(C)$ and $d(y,C) \ge d(y,C_1) $. Moreover, $d(y,C_1) > \gamma\, \diam_d(C_1)$ since $C_1 \in H$. Therefore:
\begin{align}
d(y,C) \ge d(y, C_1) > \gamma\, \diam_d(C_1) \ge \gamma\, \diam_d(C)
\end{align}
proving that $d(y, C) > \gamma\, \diam_d(C)$. Since this holds for all $y \in X \setminus C$, we have $C \in H$ as well. Therefore, $H=I(H)$.

Since $H=I(H)$, then $\vcdim(H,X)=\vcdim(I(H),X)$. Now, we use the following results of~\citet{Kivinen95}:
\begin{definition}
\label{def:sldim}[\citet{Kivinen95}, Definition 5.11]
Let $\scX$ be any domain and $\scH \subseteq 2^{\scX}$ be a concept class.
We say that $\scH$ slices $X \subset \scX$ if, for each $x \in X$, there is $h \in \Hyp$ such that $X \cap h = X \setminus \{x\}$. The slicing dimension of $(\Hyp,\scX)$, denoted by $\sldim(\Hyp,\scX)$, is the maximum size of a set sliced by $\Hyp$. If $\Hyp$ slices arbitrarily large sets, then we let $\sldim(\Hyp,\scX)=\infty$.
\end{definition}
\begin{lemma}[\citet{Kivinen95}, Lemma 5.19]\label{lem:sldim_vcdim}
If $\sldim(\Hyp,\scX) < \infty$, then $\vcdim(I(\Hyp),\scX) \le \sldim(\Hyp,\scX)$.
\end{lemma}
We will now show that $\sldim(H,X) \le M$, where $M$ is finite by assumption. By Lemma~\ref{lem:sldim_vcdim}, this will imply $\vcdim(H,X) \le \sldim(H,X)$. 

Let $S \subseteq X$ be any set of size $\sldim(H,X)$ that is sliced by $H$. To show that $\sldim(H,X) \le M$, we need to show that $|S| \le M$. If $|S| \le 2$, then this is trivial, since $M \ge 2$. Suppose then that $|S| \ge 3$. Choose $a,b \in S$ such that $d(a,b)=\diam_{d}(S)$. Since $S$ is sliced by $H$, for any $x \in S$ we must have $S \setminus x = S \cap C$ for some $C \in H$. Since $S \setminus x \subseteq C$, we have $d\big(S \setminus x, x\big) \ge d(C, x)$ and $\diam_{d}(C) \ge \diam_d(S \setminus x)$. Moreover, $d(C, x) > \gamma \, \diam_{d}(C)$, since $x \in X \setminus C$ and $C \in H$. Therefore:
\begin{align}
d\big(S \setminus x, x\big) &\ge d(C, x)
\\ &> \gamma \, \diam_{d}(C)
\\ &\ge \gamma \, \diam_{d}(S \setminus x)
\end{align}
Hence, $d(S \setminus x, x) > \gamma \, \diam_{d}(S \setminus x)$ for all $x \in S$.

Now, suppose first that $\gamma \ge 1$. Since $|S| \ge 3$, any $x \in S \setminus \{a,b\}$ yields the absurd:
\begin{align}
    \diam_{d}(S)
    &\ge d(S \setminus x, x) \\
    &> \diam_{d}(S \setminus x) && \text{ since } \gamma \ge 1 \\
    &= d(a,b) && \text{ since } a,b \in S \setminus x\\
    &= \diam_{d}(S) && \text{ by the choice of } a,b
\end{align}
Hence we must have $|S| \le 2$, which implies again $|S| \le M$. Suppose instead that $\gamma < 1$. Then, for any two distinct points $x,y \in S$, we have $d(x,y) > \gamma \diam_d(S)$. This is trivially true if $x=a$ and $y=b$; otherwise, assuming $x \notin \{a,b\}$, it follows by the fact that $d(x,y) \ge d(S \setminus x, x) > \gamma \diam_{d}(S \setminus x) = \gamma \diam_d(S)$, as seen above. Moreover, $S$ is contained in the closed ball $B(x,\diam_{d}(S))$ for any $x \in S$. Therefore, $\PackNum(B(x,r), \gamma r, d) \ge |S|$ for $r = \diam_{d}(S)$ and some $x \in \scX$. By definition of $M(\gamma,d)$ this implies that $|S| \le M(\gamma,d)$, and $M(\gamma,d) \le M$. This concludes the proof.

\section{Proof of Theorem~\ref{thm:ova_margin}}
\ovatheorem*
For the lower bounds, let $\gamma'=2 \gamma$. By definition of $M(\gamma')$, there exists a set $X$ of $M(\gamma')$ points that, according to some $d \in \{d_1,\ldots,d_k\}$, lies within a ball of radius $r > 0$, and thus has diameter at most $2 r$, and such that $d(x,y) > \gamma' r = 2 \gamma r$ for all distinct $x,y \in X$. Now choose $x \in X$ uniformly at random, and define $\C=(x, X \setminus x)$. The argument above shows that $d(x, X \setminus x) > \gamma \diam_d(X)$, which implies that $\C$ has one-versus-all margin $\gamma$. Clearly, in expectation over the distribution of $\C$, any exact cluster recovery algorithm must make $\Omega(|X|)=\Omega(M(2\gamma))$ queries. By Yao's principle for Monte Carlo algorithms, then, any such algorithm makes $\Omega(M(2\gamma))$ queries on some instance.

Let us turn to the upper bounds. The pseudocode of \MarginAlgo\ is given below. As in the proof sketch, for each $i \in [k]$ the class $H_i$ is defined as:
\begin{align}\label{eq:Hi}
   H_i = \left\{C \subseteq X \,:\, d_i(X \setminus C, C) > \gamma \, \diam_{d_i}(C) \right\}
\end{align}
As said in the proof sketch, by Lemma~\ref{lem:H_properties} and by the definition of learning with one-sided error (Definition~\ref{def:oneside}), we have $\hat{C}_i \subseteq C_i$ and therefore \MarginAlgo\ never misclassifies any point, and moreover we have:
\begin{align}
    \Pr\left( \big|C_i \setminus \hat{C}_i\big| \le \epsilon |X| \right) \ge 1 - \delta
\end{align}
In our case, that is, with $\epsilon = \nicefrac{1}{2k}$ and $\delta = \nicefrac{1}{2}$, and since $|C_i \setminus \hat{C}_i| = |C_i| - |\hat{C}_i|$, this yields:
\begin{align}
    \Pr\left( \big|\hat{C}_i\big| \ge |C_i| - \frac{|X|}{2k} \right) \ge \frac{1}{2}
\end{align}
which, since $\big|\hat{C}_i\big| \ge 0$, implies:
\begin{align}
    \E \big|\hat{C}_i\big| \ge \frac{1}{2} \cdot \left(|C_i| - \frac{|X|}{2k}\right) =  \frac{|C_i|}{2} - \frac{|X|}{4k}
\end{align}
By summing over all $i\in [k]$, at each round \MarginAlgo\ correctly labels, and removes from $X$, an expected number of points equal to:
\begin{align}
    \E \left|\hat{C}_1 \cup \ldots \cup \hat{C}_k\right| = \sum_{i=1}^{k} \E \big|\hat{C}_i\big| \ge \sum_{i=1}^{k} \left( \frac{|C_i|}{2} - \frac{|X|}{4k} \right) = \frac{|X|}{2} - \frac{|X|}{4} = \frac{|X|}{4}
\end{align}
Thus, at each round \MarginAlgo\ gets rid of an expected fraction $\nicefrac{1}{4}$ of all points still in $X$. By a standard probabilistic argument this implies that, with high probability, all points are correctly labeled within $\scO(\log n)$ rounds, see Lemma 3 of~\citep{BCLP20}. Therefore, the total number of queries is bounded by $\scO(M k \log k)$ with high probability. This concludes the proof.

\begin{algorithm}
\caption{
\label{alg:mrec}
\MarginAlgo$(X,O)$
}
\begin{algorithmic}
\If{$X = \emptyset$}{ \Return}
\EndIf
\State for each $i \in [k]$ let $H_i$ as in Equation~\ref{eq:Hi}
\State draw a sample $S$ of $|S| = \Theta\big(M k \ln k \big)$ points u.a.r.\ from $X$
\State use $O$ to learn the labels of $S$
\For{$i \in [k]$}
\State let $S_i$ be the subset of $S$ having label $i$
\State let $\hat{C}_i$ be the smallest set in $H_i$ consistent with $S_i$
\State give label $i$ to every $x \in \hat{C}_i$
\EndFor
\State let $X' = X \setminus \cup_{i \in [k]} \hat{C}_i$
\State \MarginAlgo$(X',O)$
\end{algorithmic}
\end{algorithm}

\section{Proof of Theorem~\ref{thm:kclass_learning}}
As said in the sketch, the proof follows the same ideas of the proof of Theorem~\ref{thm:ova_margin}.

For the lower bounds, suppose that $\cosl(\Hyp) < \infty$, and let $X$ be a set cosliced by $\Hyp$ with $|X|=\cosl(\Hyp)$.  We draw a random uniform element $x \in X$, and we consider the clustering $\C=(x, X \setminus x)$. By definition of sliced set, $\C$ is realised by $\Hyp$, and therefore it satisfies the assumptions. Now the same arguments of the lower bounds of Theorem~\ref{thm:ova_margin} imply that any algorithm needs $\Omega(|X|)=\Omega(\cosl(\Hyp))$ queries on some instance to return $\C$. Clearly, if $\cosl(\Hyp) = \infty$ this means that we can take $|X|=n$ arbitrarily large, whence the second lower bound.

For the upper bounds, let $P_k(X)$ be the set of all $k$-clusterings of $X$ that are realised by $\Hyp$. Then, for each $i \in [k]$ we define:
\begin{align}\label{eq:HiHi2}
    H_i = \left\{C' \,:\, C' = C'_i \,\wedge\, (C_1',\ldots,C_k') \in P_k(X) \right\}
\end{align}
As in \MarginAlgo, we learn each class $H_i$ with one-sided error by choosing the smallest hypothesis in $I(H_i)$ that is consistent with $S_i$, for a labeled sample $S$ of size $\Theta(\vcdim(I(H_i),X) \, k \ln k)$. As shown in the proof of Lemma~\ref{lem:H_properties}, a result of~\citet{Kivinen95} implies that if $\sldim(H_i,X) < \infty$ then $\vcdim(I(H_i),X) \le \sldim(H_i,X)$. Therefore, to prove the theorem we only need to show that $\sldim(H_i,X) \le \cosl(\Hyp)$. To this end, suppose that $U=\{x_1,\ldots,x_{\ell}\} \subseteq X$ is sliced by $H_i$. By construction of $H_i$, this means that there are $\ell$ clusterings $\C_1,\ldots,\C_{\ell}$, each one realised by $\Hyp$, such that $\C_i = (\{x_i\}, U \setminus \{x_i\})$ for all $i \in [k]$. This implies that $U$ is cosliced by $\Hyp$. Hence, $|U| \le \cosl(\Hyp)$ and so $\sldim(H_i,X) \le \cosl(\Hyp)$, as claimed. The rest of the proof is similar to the proof of Theorem~\ref{thm:ova_margin}, and shows that at each round we recover an expected constant fraction of all points, and that therefore all points will be recovered with high probability after $\scO(\log n)$ rounds. This shows that we can recover $\C$ by making with high probability at most $\scO(\cosl(\Hyp) k \log k \log n)$ queries.

\section{Proof of Theorem~\ref{thm:affine}}
\affinetheorem*
The upper bound follows from Theorem~\ref{thm:ova_margin}, by using the standard fact that in the Euclidean metric the unit ball has covering number $\CoverNum(B(x,1),\epsilon) \le (1+\nicefrac{2}{\epsilon})^m$ for all $\epsilon > 0$. As it is well-known, we have $\PackNum(B(x,1),\epsilon) \le \CoverNum(B(x,1),\nicefrac{\epsilon}{2})$, which for $\epsilon=\gamma$ yields $M(\gamma) \le (1+\nicefrac{4}{\gamma})^m$.

We now prove the lower bound.
Having instances with ``arbitrarily small one-versus-all margin''  means that $\gamma = 0$, see Definition~\ref{def:ova_margin}. Take any $h \in \Hyp$ such that both $h$ and its complement $\bar h$ contain a ball of positive radius. Note that this implies that, for any $\rho > 0$, there exists a closed ball $B$ with radius $r>0$ such that:
\begin{align}
    B \subseteq h \\
    \exists x \in \bar h \,:\, d(B,x) \le \rho
\end{align}
Let $c$ be the center of $B$. Consider a sphere $S$ of radius $r'$ that contains $x$ and whose center $c'$ lies on the affine subspace $x+\alpha(c-x)$. Let $\eta = \sup_{y \in S \setminus B} d(x, y)$ and let $X$ be an $\eta$-packing of $S$. Note that, since $\gamma=0$, we can choose $\rho$ and $r'$ arbitrarily small, and in particular we can make the ratio $\frac{\eta}{r'}$ arbitrarily small. This implies that we can make $X$ arbitrarily large, see Figure~\ref{fig:packing}.
\begin{figure}[h]
    \centering
    \begin{tikzpicture}[scale=.5,every node/.style={circle,inner sep=.8pt}]
    \node[draw] (orig) at (0,0) {};
    \node[above] (origtxt) at (orig) {$c$};
    \draw[draw=gray,fill=gray,fill opacity=.2] (0,0) circle (3.5);
    \node (B) at (3,-2.4) {$B$};
    \node[draw] (c) at (-2.6,0) {};
    \node[above] (cxt) at (c) {$c'$};
    \def\r{1}
    \def\N{9}
    \pgfmathsetmacro\ang{360/\N}
    \draw[densely dotted] (c) circle (\r);
    \foreach \i in {1,...,\N} {
        \node[draw,fill] (x) at ($(c)-(\ang*\i:\r)$) {};
    }
    \node[left] (xtx) at ($(x)$) {$x$};
    \node[right] (Stxt) at ($(x)+(1.6*\r,-\r)$) {$X$};
    \end{tikzpicture}
    \caption{An $\eta$-packing $X$ such that $X \cap h = X \setminus \{x\}$ and $X \cap \bar h = \{x\}$. The ball $B$ is by construction entirely in $h$, whereas $x$ is by construction in $\bar h$.}
    \label{fig:packing}
\end{figure}
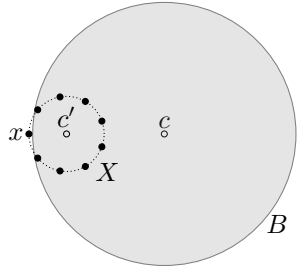

Now, consider $x' \in X \setminus \{x\}$. As by construction $d(x',x) > \eta$, and as $r' < r$, we must have $x' \in B$. Therefore, $x' \in h$. Hence the concept $h_x = h$ is such that $X \cap h_x = X \setminus \{x\}$. Now, for any $x' \in X \setminus \{x\}$, there is a rotation $R$ with fixed point $c'$ and such that $R(x')=x$. Hence, $h_{x'} = R^{-1} h_x$ is such that $X \cap h_{x'} = X \setminus \{x'\}$. Since $R^{-1}$ is an affine transformation, $h_{x'} \in \Hyp$ as well. Hence, for every $x \in X$ there exists some concept $h_x \in \Hyp$ such that $X \cap h_x = X \setminus \{x\}$.

Now consider the complement $\bar{h}$ of $h$. Note that $\bar h \in \co(\Hyp)$. The first part of the argument above can be applied to $\bar h$ as well, showing that for every $x \in X$ there exists $\bar{h_x} \in \co(\Hyp)$ such that $X \cap \bar{h_x} = X \setminus \{x\}$. Now consider the complement $h_x$ of $\bar{h_x}$. Clearly $h_x \in \Hyp$, and moreover, $X \cap h_x = \{x\}$.

Hence, for any $x \in X$ there are two concepts $h_x^-,h_x^+ \in \Hyp$ such that $X \cap h_x^- = X \setminus \{x\}$ and $X \cap h_x^+ = \{x\}$. Hence, every $2$-clustering $\C$ of $X$ in the form $C_1=\{x\}, C_2=X \setminus \{x\}$ is realized by $\Hyp$. Note that this holds with $X$ fixed; we just need to transform the concetps appropriately. It is they immediate to see that any algorithm must perform $\Omega(|X|)$ queries on some instance $(X,O)$. As we can make $X$ arbitrarily large, this completes the proof.

\end{document}